\documentclass[10pt]{article} 
\usepackage{amsmath,amssymb,amsthm,mathtools}
\usepackage[normalem]{ulem}
\usepackage{newtxtext}
\usepackage[margin=1in]{geometry}
\usepackage{booktabs}
\usepackage{algorithmic}
\usepackage{enumitem}
\usepackage{xcolor}
\usepackage[ruled,vlined]{algorithm2e}

\usepackage{amsmath,amsfonts,bm}

\newcommand{\dd}{{\rm d}}

\def\1{\bm{1}}

\newcommand\independent{\protect\mathpalette{\protect\independenT}{\perp}}
\def\independenT#1#2{\mathrel{\rlap{$#1#2$}\mkern2mu{#1#2}}}

\DeclareMathAlphabet{\mathsfit}{\encodingdefault}{\sfdefault}{m}{sl}
\SetMathAlphabet{\mathsfit}{bold}{\encodingdefault}{\sfdefault}{bx}{n}

\def\gE{{\mathcal{E}}}

\def\gK{{\mathcal{K}}}
\def\gL{{\mathcal{L}}}

\def\gN{{\mathcal{N}}}

\def\gY{{\mathcal{Y}}}

\newcommand{\E}{\mathbb{E}}

\newcommand{\R}{\mathbb{R}}
\newcommand{\N}{\mathbb{N}}

\newcommand{\KL}{D_{\mathrm{KL}}}
\newcommand{\Var}{\mathrm{Var}}

\newcommand{\Cov}{\mathrm{Cov}}

\usepackage[backref,colorlinks,citecolor=blue,bookmarks=true]{hyperref}
\usepackage{mathtools, amssymb, amsthm, dsfont}
\numberwithin{equation}{section}

\usepackage{parskip}
\SetAlCapSkip{1em}
\usepackage{mathtools, amssymb, amsthm, bbm}
\usepackage[capitalize]{cleveref}

\theoremstyle{plain}
\newtheorem{theorem}{Theorem}[section]
\newtheorem{lemma}[theorem]{Lemma}
\newtheorem{corollary}[theorem]{Corollary}

\newtheorem{assumption}[theorem]{Assumption}

\theoremstyle{definition}

\theoremstyle{plain}

\newtheorem{remark}[theorem]{Remark}

\newcommand{\norm}[1]{\|#1\|}

\usepackage{url}
\usepackage[title]{appendix}

\title{Provably Reliable Classifier Guidance via \\ Cross-Entropy Control}

\author{ Sharan Sahu\footnotemark[1] \\ \texttt{ss4329@cornell.edu} \\ Department of Statistics and Data Science \\ Cornell University \and Arisina Banerjee\footnotemark[1] \\ \texttt{ab2957@cornell.edu} \\ Department of Statistics and Data Science \\ Cornell University \and Yuchen Wu \\ \texttt{yuchen.wu@cornell.edu} \\ School of Operations Research and Information Engineering \\ Cornell University }

\makeatletter
\gdef\equalcontrib{
    \renewcommand{\thefootnote}{\fnsymbol{footnote}} 
    \footnotetext[1]{\normalsize These authors contributed equally to this work.}
    \renewcommand{\thefootnote}{\arabic{footnote}}
}
\makeatother

\begin{document}

\maketitle
\equalcontrib
\begin{abstract}
Classifier-guided diffusion models generate conditional samples by augmenting the reverse-time score with the gradient of the log-probability predicted by a probabilistic classifier. 
In practice, this classifier is usually obtained by minimizing an empirical loss function. 
While existing statistical theory guarantees good generalization performance when the sample size is sufficiently large, it remains unclear whether such training yields an effective guidance mechanism.

We study this question in the context of cross-entropy loss, which is widely used for classifier training. 
Under mild smoothness assumptions on the classifiers, we show that controlling the cross-entropy at each diffusion model step is sufficient to control the corresponding guidance error. 
In particular, we show that probabilistic classifiers achieving conditional KL divergence $\varepsilon^2$ induce guidance vectors with mean squared error $\widetilde O(d \varepsilon )$, up to constant and logarithmic factors.
\textcolor{black}{We demonstrate that the proposed smoothness condition is necessary, by constructing a sequence of non-smooth classifiers that achieve small conditional KL while inducing an exploding guidance error. 
We also show that the derived upper bound $\widetilde O(d \varepsilon )$ is optimal up to poly-logarithmic factors.
}
Our result yields an upper bound on the sampling error of classifier-guided diffusion models and bears resemblance to a reverse log-Sobolev--type inequality.
To the best of our knowledge, this is the first result that quantitatively links classifier training to guidance alignment in diffusion models, providing both a theoretical \textcolor{black}{
understanding of when such models succeed, along with principled guidelines for selecting classifiers that induce effective guidance.} 
\textcolor{black}{In particular, our findings suggest that, when selecting classifiers for guidance, one should prioritize not only accuracy but also smoothness, highlighting the advantages of smoothness-inducing training methods for diffusion guidance. }

\end{abstract}

\section{Introduction}


\emph{Score-based diffusion models} have become a leading approach to generative modeling, achieving high-quality sample generation across a wide range of applications \cite{sohl2015deep,song2019generative,ho2020denoising,songdenoising,songscore,chen2024overview}. 
At a high level, diffusion models generate samples via a sequential denoising process driven by a sequence of time-dependent \emph{score functions},  which in practice are usually learned from data via the score matching approach \cite{hyvarinen2005estimation}.

Among the many extensions of this framework, \emph{classifier guidance} has become a standard technique in modern diffusion models, 
enabling controllable and high-fidelity conditional sample generation across diverse domains \cite{dhariwal2021diffusion, Ber_Diffusion_MICCAI2024, wang2025guidpaintclassguidedimageinpainting, 44cc1b7eabab4bf68460843575332bc9}. 
Specifically, to generate samples conditioned on a label $y$, the classifier guidance approach trains a probabilistic classifier to approximate the time-dependent label posterior $p_t(y \mid x)$. 
During sampling, a classifier-guided diffusion model augments the unconditional score $\nabla \log p_t(x)$ with the gradient of the log-probability predicted by the learned classifier, thereby steering the generative process toward the desired condition. 
In practice, the classifier used for diffusion guidance is typically learned via empirical risk minimization. 
The most commonly used loss function in this setting is the cross-entropy loss, which is closely related to the  Kullback-Leibler (KL) divergence between the true conditional label distributions given the covariates and the classifier's predictions. 

A large body of statistical learning theory guarantees that, with sufficiently large sample sizes, the trained classifier generalizes well, achieving low classification error or, equivalently, a small conditional KL divergence between the true conditional label distribution and the estimated conditional label distribution \cite{hastie2009elements}.
However, low classification error alone does not guarantee that a trained classifier yields an effective guidance algorithm, as generalization under the cross-entropy loss does not necessarily imply that the gradient of the log-predicted probability aligns with that of the true conditional probability.
In fact, empirical studies suggest that classifier guidance can be brittle: even when the classifier predicts labels correctly, its gradients can point in the wrong direction, leading to oversaturated or semantically inconsistent samples \cite{Wallace2023EndtoEndDL, pmlr-v202-ganz23a, pmlr-v202-dinh23a, dinh2023rethinking}. 
This observation reveals a fundamental theoretical gap: classifier training objectives are typically formulated in terms of \emph{density divergences} (e.g., conditional KL), whereas the diffusion model sampling dynamics depend on the accuracy of a \emph{vector field} $\nabla_x \log p_t(y \mid x)$.
In many applications, however, guarantees are available only in terms of the training loss, hence motivating the following key questions:
\begin{enumerate}
    \item[Q1] \textcolor{black}{Can we identify the precise conditions under which a small conditional label KL divergence (equivalently, cross-entropy) guarantees effective classifier guidance?}
    \item[Q2] \textcolor{black}{Under the above conditions, how can the sampling error be upper bounded as a function of the conditional KL divergence? 
    }
\end{enumerate}
\textcolor{black}{Answering these questions not only provides a theoretical understanding of classifier guidance, but also offers practical guidance for selecting classifiers: one should favor classifiers that satisfy the conditions in Q1, so that strong classification performance translates into effective guidance. }


\subsection{Our contributions}

In this work, we study the theoretical properties of classifier-guided diffusion models \textcolor{black}{by addressing both Q1 and Q2. } 
In particular, we make the following contributions:

\paragraph{Small conditional KL does not necessarily ensure effective guidance. } 
We first show \textcolor{black}{through a counterexample} that probabilistic classifiers with small conditional KL may fail to provide effective diffusion guidance.
To illustrate this, we construct a sequence of classifiers that attain vanishingly small conditional KL, while the associated guidance errors remain bounded away from zero or even tend to infinity. 
Our construction of classifiers is based on adding high-frequency perturbations to the true conditional label distributions, and disproves the belief that small classification error alone guarantees effective classifier guidance.

\paragraph{Small conditional KL implies effective  guidance under classifier smoothness conditions.}
On the other hand, we show that  probabilistic classifiers that achieve small conditional KL ensure effective guidance, 
\textcolor{black}{provided they satisfy certain smoothness conditions that are also satisfied by the ground-truth model.} 
Specifically, 
\textcolor{black}{under mild regularity conditions we show that}
a conditional KL divergence of $O(\varepsilon^2)$ between the target label conditionals and the classifier-predicted conditionals implies a guidance MSE of $\widetilde O(d\varepsilon)$, up to constant and logarithmic factors. 
We further show that \textcolor{black}{the derived upper bound} 
is \emph{optimal} by constructing an explicit instance that attains this rate.

 \paragraph{Implications for sampling with classifier-guided diffusion models.}
As a \textcolor{black}{direct} application of our result, we use it to bound the sampling error in classifier-guided diffusion models. Specifically, we show that for Denoising Diffusion Probabilistic Models (DDPMs) with classifier guidance, if the classifiers at each diffusion model step satisfy the same smoothness conditions as the corresponding label conditionals, then the KL divergence between the target and output distributions by this diffusion model scales as $\widetilde O(d \varepsilon_{\rm average})$ (ignoring errors from score estimation and discretization), where $\varepsilon_{\rm average}$ is the average of $\varepsilon_t$, with $\varepsilon_t^2$ denoting the conditional KL achieved by the classifier at time $t$.



\subsection{Related work}
\paragraph{Score-based generative models.}
This line of research traces back to the principle of score matching \cite{JMLR:v6:hyvarinen05a}, 
which seeks to learn a target distribution $p_0(x)$ via estimating its score function $\nabla \log p_0(x)$. 
Subsequent work has shown that score matching corresponds to minimizing the Fisher divergence over some function classes, 
and in certain statistical models is closely related to maximum likelihood estimation \cite{lyu2012interpretation,sriperumbudur2017density,koehlerstatistical}.
Leveraging score matching, score-based generative models, in particular diffusion models, have emerged as a powerful class of generative models, attaining state-of-the-art performance in numerous applications \cite{podellsdxl,betker2023improving,brooks2024video}. 
Representative examples of score-based generative models include Denoising Diffusion Probabilistic Models (DDPM) \cite{ho2020denoising} and Denoising Diffusion Implicit Models (DDIM) \cite{songdenoising}. 
Various theoretical results have been established in attempt to explain the success of score-based generative models, focusing on different perspectives ranging from convergence \cite{lee2023convergence,chen2023sampling,bentonnearly,gupta2024faster,li2024unified,wu2024stochastic,liang2024broadening,huang2025convergence,li2025faster,citation-0}, score learning \cite{wibisono2024optimal,oko2023diffusion,zhang2024minimax,dou2024optimal,koehler2023statistical,wang2024evaluating}, to statistical applications \cite{el2022sampling,montanari2023posterior,xu2024provably,mei2025deep,chen2025diffusion}. We refer interested readers to \cite{chen2024overview} for an overview of recent theoretical advances in diffusion models.
In these examples, sample generation relies on accurately approximating the score functions of progressively noised data distributions, given by $\nabla \log p_t(x)$ at a collection of time points $t$.

\paragraph{Classifier-guided diffusion models.}

As discussed in the previous sections, classifier guidance \cite{dhariwal2021diffusion} is a technique that steers the sample generation of an unconditional diffusion model toward the conditional distribution associated with a class label $y$, via augmenting the base model's score with the gradient log-probability of a separately trained, time-dependent probabilistic classifier $\nabla_x \log \hat{p}_{t}(y \mid x)$. 
The effectiveness of this framework critically depends on the quality of the guidance vector $\nabla_x \log \hat{p}_{t}(y \mid x)$, which ideally approximates the ground-truth guidance vector $\nabla_x \log p_t(y \mid x)$.
Notably, several works analyze the behavior of guided diffusion models assuming access to the true guidance vectors \cite{wu2024theoretical,bradley2024classifier,chidambaram2024does,li2025provable,jiao2025towards}.

Recent work has explored, both empirically and theoretically, the role of classifier quality in the design of effective guidance algorithms.
For example, \cite{ma2024elucidating} empirically observes that standard off-the-shelf classifiers are often poorly calibrated, and argue that classifier smoothness is crucial for reliable gradient estimation. On the theoretical side, they establish that $L^2$ convergence of densities implies $L^2$ convergence of the corresponding guidance vectors under suitable smoothness assumptions. 
Their analysis requires the densities to be bounded above and below on a compact and convex set $\Omega$, and the estimators to belong to a Sobolev space $H^k(\Omega)$ with smoothness $k>1$. 
However, the lower bound assumption could be restrictive, as it precludes regions where the density is vanishingly small.
Another work \cite{oertell2025efficientcontrollablediffusionoptimal} focuses on the binary classification setting, assuming that the conditional probabilities are smooth and follow a specific form determined by a bounded reward function. 
Their classifier construction relies critically on knowledge of this reward function.
More recently, several papers analyze diffusion guidance from different perspectives. 
For example, \cite{guo2024gradient} recasts diffusion guidance as solving a regularized optimization problem along the diffusion trajectory, 
and \cite{tang2024conditional} develops a stochastic-analysis framework for understanding and analyzing diffusion guidance. 
On the application side, \cite{sun2024rectifid} shows that classifier guidance remains powerful in rectified-flow models when coupled with a fixed-point formulation and anchored trajectories.

\paragraph{Information-theoretic divergences.}
From a theoretical perspective, our work is closely related to the connection between the KL divergence and the Fisher divergence. 
The relationship between these two divergences has long been a central topic in information theory and probability. 
For example, de Bruijn's identity relates the \emph{time-derivative} of the KL divergence (under a Gaussian convolution) to the Fisher information and has been extended and refined in various settings (e.g., \cite{23c0155fc92345fca7d191676fbb9ff6}). 
Beyond the infinitesimal regime, a growing literature studies divergences built directly from Fisher information. In particular, \cite{sanchezmoreno2012jensen} introduces the Jensen-Fisher divergence that parallels the Jensen-Shannon divergence but with Fisher information in place of entropy, while \cite{kharazmi2022cumulative} and follow-up works develop generalized and $q$-Fisher information measures together with their associated Jensen-type divergences. 
\cite{sriperumbudur2017density} analyze density estimation in (possibly infinite-dimensional) exponential families under Fisher divergence, emphasizing that control of Fisher divergence is strictly stronger than control of KL or total variation in many natural settings. 
In addition, if a distribution satisfies a logarithmic Sobolev inequality, then the associated KL divergence can be bounded in terms of the relative Fisher information.
Our paper addresses a different question: we ask under what conditions a small conditional KL divergence over the label distribution implies a small MSE of the guidance vector, which is related to, but not identical to the Fisher divergence. This question is closely related in spirit to the {reverse} log-Sobolev inequalities studied in the functional-inequality literature.
For instance, \cite{eldan2018gaussianwidthgradientcomplexityreverse, eldan2019dimensionfreereverselogarithmicsobolev} show that, for Gaussian reference measures and distributions of low ``gradient complexity'', one can derive almost-tight reverse log-Sobolev inequalities that control Fisher information by entropy.




\paragraph{Classifier-free guidance.}
A popular alternative to classifier-guided approach is classifier-free guidance \cite{ho2021classifierfree}. 
This method eliminates the need for a separately trained classifier by jointly training a single diffusion model on both conditional and unconditional objectives, via randomly dropping the conditioning information during training.
At sampling time, guidance is achieved by steering the process using a scaled difference between the conditional and unconditional score estimates produced by this single model.


\subsection{Organization}

The rest of this paper is organized as follows.
Section \ref{sec:preliminaries} introduces some preliminaries
regarding diffusion models and classifier guidance. 
We state in Section \ref{sec:main} our main theoretical results, and we present in Section \ref{sec:experiments} numerical experiments that support our theoretical findings. 

\section{Preliminaries}
\label{sec:preliminaries}
\subsection{Notation}
For $n \in \N_+$, define $[n] = \{1, 2, \cdots, n\}$. 
For $r>0$, we denote the closed Euclidean ball of radius $r$ by $\mathcal{B}_r = \{ x \in \R^d : \|x\|_2 \le r \}$,  and the unit sphere by $\mathbb{S}^{d-1} = \{ x \in \R^d : \|x\|_2 = 1 \}$. 
We write $C^k(\R^d)$ for the set of $k$-times continuously differentiable functions $f:\R^d\to\R$.  
For two positive sequences $\{a_n\}_{n = 1}^{\infty}$ and $\{b_n\}_{n = 1}^{\infty}$, 
we write $a_n \lesssim b_n$ if there exists a constant $C>0$ that is independent of $n$, such that $a_n \le Cb_n$. 
We write $a_n \asymp b_n$ if $a_n \lesssim b_n$ and $b_n \lesssim a_n$. 
We write $a_n = O(b_n)$ if $a_n \le Cb_n$ for some constant $C > 0$ that is independent of $n$, we write $a_n = \Omega(b_n)$ if $b_n = O(a_n)$, and write $a_n = \Theta(b_n)$ if $a_n = O(b_n)$ and $b_n = O(a_n)$. 
We write $a_n = \widetilde O(b_n)$ to mean that there exists a constant $C>0$ that is independent of $n$, and a poly-logarithmic factor $\mathrm{polylog}(\Pi)$ in the relevant problem parameters $\Pi$, such that $|a_n| \le C\,b_n \cdot \mathrm{polylog}(\Pi)$. 
We use $\nabla_x \cdot$ to represent the divergence operator with respect to the variable $x$.

\subsection{Information-theoretic quantities}

For two distributions $P$ and $Q$, the KL divergence between $P$ and $Q$ is 
\begin{align*}
    \KL(P \parallel Q) = \int P(\dd x) \log \frac{P(\dd x)}{Q(\dd x)}. 
\end{align*}
The total variation (TV) distance between $P$ and $Q$
 is  
\begin{align*}
    \mathrm{TV}(P,Q) = \sup_A |P(A) - Q(A)|. 
\end{align*}
The cross-entropy between $P$ and $Q$ is 
\begin{align*}
    H(P, Q) = - \E_P[\log Q]. 
\end{align*}
Note that $H(P, Q) = H(P) + \KL(P \parallel Q)$, where $H(P)$ denotes the entropy of $P$. 
Closely related to cross-entropy is the cross-entropy loss $\mathcal{L}_{CE}$. 
Specifically, for a covariate $x$ and a label $y \in \mathcal{Y}$, 
the cross-entropy loss associated with a probabilistic classifier $\hat p$ is 
\begin{align*}
    \mathcal{L}_{\rm CE} = -\sum_{y' \in \mathcal{Y}} \mathbbm{1}_{y = y'} \log \hat p(y' \mid x). 
\end{align*}
When $(x, y) \sim p$, we have 
\begin{align*}
    \E[\mathcal{L}_{\rm CE}] =\E_{x \sim p}\big[ H(p(\cdot \mid x), \hat p(\cdot \mid x)) \big]. 
\end{align*}
Therefore, minimizing the empirical cross-entropy loss often leads to a small conditional KL divergence between the true and estimated label conditionals. 




\subsection{Diffusion models through a continuous-time perspective}
\label{subsec:ou-diffusion}

Diffusion models are often studied via their continuous-time limits, and are formulated as the numerical approximations to these limits.  
Conceptually, diffusion models consist of a forward and a reverse process. A common choice for the forward process is based on the Ornstein-Uhlenbeck (OU) process.
Specifically, let the data distribution $P_{\mathrm{data}}$ be a probability
measure over $\R^d$. 
The OU process $(\overrightarrow{X}_t)_{t\in[0,T]}$ is defined as the solution to the following SDE:
\begin{equation}
  d\overrightarrow{X}_t = -\overrightarrow{X}_t\,dt + \sqrt{2}\,dB_t,
  \qquad \overrightarrow{X}_0 \sim P_{\mathrm{data}},
  \label{eq:ou-forward}
\end{equation}
where $(B_t)_{t\in[0,T]}$ is a standard Brownian motion in $\R^d$. Process \eqref{eq:ou-forward} has an explicit solution:
\[
  \overrightarrow{X}_t = \lambda_t \overrightarrow{X}_0 + \sigma_t Z,
  \qquad Z \sim \gN(0,I_d) \;\;\mbox{and}\;\; \overrightarrow{X}_0 \independent Z,
\]
where $\lambda_t = e^{-t}$ and $\sigma_t^2 = 1 - e^{-2t}$. 
Conditioned on $\overrightarrow{X}_0=x_0$, we therefore have the Gaussian transition kernel
\[
  \overrightarrow{X}_t \mid \overrightarrow{X}_0 = x_0
  \sim \gN\big(\lambda_t x_0, \sigma_t^2 I_d\big),
  \qquad
  p_t(x_t \mid x_0)
  = (2\pi\sigma_t^2)^{-d/2}
    \exp\Big(
      -\frac{\|x_t - \lambda_t x_0\|_2^2}{2\sigma_t^2}
    \Big).
\]
With a slight abuse of notations,
we also use $p_t$ to denote the marginal density of $\overrightarrow{X}_t$, and we define
\[
  m_t(x_t) = \E[\overrightarrow{X}_0 \mid \overrightarrow{X}_t = x_t] \in \R^d,
  \qquad
  \Sigma_t(x_t) = \Cov[\overrightarrow{X}_0 \mid \overrightarrow{X}_t = x_t] \in \R^{d \times d}
\]
as the posterior mean and covariance of $\overrightarrow{X}_0$ given $\overrightarrow{X}_t=x_t$.
The following equation is a straightforward consequence of Tweedie's formula \cite{robbins1992empirical}:
\begin{equation*}
  \nabla_x \log p_t(x)
  = -\sigma_t^{-2} x + \lambda_t \sigma_t^{-2} m_t(x),
  \qquad x\in\R^d.
  \label{eq:tweedie-unconditional}
\end{equation*}
Likewise we define the time-reversal $(\overleftarrow{X}_t)_{t\in[0,T]}$ of process \eqref{eq:ou-forward} by setting $\overleftarrow{X}_t = \overrightarrow{X}_{T-t}$,  so that $\overleftarrow{X}_t \sim p_{T - t}$. 
Under mild regularity assumptions on the target distribution, the process
$(\overleftarrow{X}_t)_{t\in[0,T]}$ satisfies the following SDE:
\begin{equation}
  \dd \overleftarrow{X}_t
  = \big\{ \overleftarrow{X}_t + 2\nabla_x \log p_{T-t}(\overleftarrow{X}_t) \big\}\,\dd t
    + \sqrt{2}\,\dd B'_t,
  \qquad \overleftarrow{X}_0 \sim p_T,
  \label{eq:reverse-sde}
\end{equation}
where $(B'_t)_{t\in[0,T]}$ is another Brownian motion in $\R^d$ that is independent of
$\overleftarrow{X}_0$ \cite{ANDERSON1982313, cattiaux2022timereversaldiffusionprocesses}. 
Sampling $\overleftarrow{X}_T$ according to process \eqref{eq:reverse-sde} yields samples from
$P_{\mathrm{data}}$. In practice, the score $\nabla_x \log p_t(x)$ is not accessible and is approximated by a neural network $s_\theta(x,t)$ trained to minimize
the following score-matching objective
\begin{equation*}
  \gL_{\mathrm{score}}(s_\theta)
  =
  \int_\delta^T
    \E_{X_t \sim p_t}
    \big[
      \|s_\theta(X_t,t) - \nabla_x \log p_t(X_t)\|_2^2
    \big]\,\dd t,
\end{equation*}
where $\delta > 0$ denotes an early stopping cutoff.
In practice, an empirical version of this objective is minimized using samples from the target distribution, for instance via score matching.
To simulate the reverse dynamics (\ref{eq:reverse-sde}), a standard approach is to discretize time on a grid
\[
  0 = t_0 < t_1 < \cdots < t_N < T,
  \qquad \tau_k = t_{k+1} - t_k.
\]
%
For sufficiently large $T$, we have $p_T \approx \pi_d$, where $\pi_d$ denotes the distribution of a $d$-dimensional standard Gaussian random vector. 
Consequently, it is natural to initialize the reverse process from $\pi_d$.
In practice, rather than tracking the reverse process all the way to time $T$, one often employs \emph{early stopping} at time $T-\delta$ for some
small $\delta>0$, 
in order to mitigate numerical instability arising from the score $\nabla_x \log p_t$ as $t\downarrow 0$.


\subsection{Conditional sampling through classifier guidance}
\label{subsec:classifier-guidance}
Classifier guidance is a standard technique in diffusion models that enables conditional sampling by augmenting the drift term of the reverse process \eqref{eq:reverse-sde} with the gradient of a trained probabilistic classifier.
Specifically, we use $Y$ to represent the label of interest, which takes values in a finite label set $\gY$ with some prior. 
The label $Y$ is correlated with $\overrightarrow{X}_0$. 
With a slight abuse of notation, we also use $P_{\rm data}$ to represent the joint distribution of $(\overrightarrow{X}_0, Y)$, and use $p_t$ to represent the joint distribution of $(\overrightarrow{X}_t, Y)$ at time $t$. 
Since the OU forward process \eqref{eq:ou-forward} operates solely on the $X$-coordinate,
the label prior remains unchanged over time, so that $p_t(y) = P_{\rm data}(y)$ for all $t \in [0, T]$. 

Our goal is to sample from the conditional distribution $\overrightarrow{X}_0 \mid Y = y$.
While one could train a separate conditional diffusion model for this task, doing so would require retraining the entire model for each new condition, which can be computationally expensive. 
Instead, classifier guidance employs a probabilistic classifier associated with the target condition to steer an unconditional diffusion model toward the desired conditional distribution, incurring substantially lower training cost.

To motivate classifier guidance, we consider the exact conditional diffusion model, whose score function at time $t$ is 
\begin{equation*}
  \nabla_x \log p_t(x\mid y)
  = \nabla_x \log p_t(x)
    + \nabla_x \log p_t(y\mid x).
\end{equation*}
The conditional score decomposes into an unconditional score plus a guidance term $\nabla_x \log p_t(y\mid x)$. 
In the idealized setting where the exact score and guidance terms are available, the  corresponding reverse SDE is given by
\begin{equation}
  \dd \overleftarrow{X}_t^{(y)}
  =
  \Big\{
    \overleftarrow{X}_t^{(y)}
    + 2\nabla_x \log p_{T-t}(\overleftarrow{X}_t^{(y)})
    + 2 \,\nabla_x \log p_{T - t}(y\mid \overleftarrow{X}_t^{(y)})
  \Big\} \dd t
  + \sqrt{2}\,\dd B'_t, \qquad X_t^{(y)} \sim p_T(\cdot \mid y).
  \label{eq:true-guided-sde}
\end{equation}
%
In practice, $p_t(y\mid x)$ is unknown and must be approximated by a trained probabilistic classifier $\hat p_{t}(y\mid x)$. 
A common objective is to minimize the conditional label KL
\begin{equation*}
  \gL_{\mathrm{cls}}(\hat p_t; t)
  =
  \E_{X_t \sim p_t}
  \Big[
    \KL\big(p_t(\cdot\mid X_t)\,\Vert\,\hat p_{t}(\cdot\mid X_t)\big)
  \Big], 
\end{equation*}
which can be approximately achieved by minimizing the empirical cross-entropy loss.
Given a trained probabilistic classifier, the classifier guidance approach uses the gradient of the log predicted probability to define an {approximate guidance vector field}, 
which is then incorporated into the classifier-guided reverse SDE together with the pre-trained score:
\begin{equation}
  \dd \hat X_t^{(y)}
  =
  \Big\{
    \hat X_t^{(y)}
    + 2  s_\theta(\hat X_t^{(y)}, T-t)
    + 2 \gamma \nabla_x \log \hat p_{T - t}(y\mid \hat X_t^{(y)})
  \Big\} \dd t
  + \sqrt{2}\,\dd B_t', \qquad \hat X_0^{(y)} \sim \pi_d,
  \label{eq:approx-guided-sde}
\end{equation}
where $\gamma \geq 0$ is a hyperparameter that controls the guidance strength. 
In practice, a discretized version of process \eqref{eq:approx-guided-sde} yields a classifier-guided diffusion model.
When setting $\gamma = 1$, process \eqref{eq:approx-guided-sde} approximates process \eqref{eq:true-guided-sde} and, with accurate drift estimates, approximately terminates at the target conditional distribution.
In this case, a standard application of Girsanov's Theorem shows that the KL divergence between processes \eqref{eq:true-guided-sde} and \eqref{eq:approx-guided-sde} is  controlled by the averaged score MSE and the averaged guidance term MSE. 
In particular, the guidance term MSE at time $t$ is given by
\begin{equation*}
  \gE_{\mathrm{guid}}(t,y)
  =
  \mathbb{E}_{X_t \sim p_t(\cdot\mid y)}
  \Big[
    \|
      \nabla_x \log p_t(y\mid X_t)
      - \nabla_x \log \hat p_{t}(y\mid X_t)
    \|_2^2
  \Big].
\end{equation*} 
Notably, in practice, practitioners typically set $\gamma > 1$ to enhance sample fidelity \cite{dhariwal2021diffusion,ho2021classifierfree} at the expense of sample diversity. In such settings, the goal is not to sample exactly from $P_{\rm data}(\cdot \mid y)$, but rather to generate visually convincing samples.
\textcolor{black}{
We note that setting $\gamma = 1$ lets \eqref{eq:approx-guided-sde} approximate the conditional distribution $P_{\mathrm{data}}(\cdot \mid y)$.}


 
\section{Main results}
\label{sec:main}

\subsection{Controlling conditional KL does not guarantee effective guidance
}
\label{subsec:kl-not-score}

We begin by showing that, in the absence of any regularity assumptions on the classifiers, a small conditional KL divergence between the true conditional label distribution and its approximation does not ensure convergence of the guidance term. 
Concretely, we construct a sequence of probabilistic classifiers $(\hat p_n)_{n\ge1}$ that achieve vanishingly small conditional KL, yet the associated guidance vector remains misaligned.

\vspace{0.5em}

\begin{theorem}
\label{thm:kl-not-score}
For any density $p$ such that $p(\{x: 0 < p(y \mid x) < 1\} \mid y) > 0$ and the map $x \mapsto p(y \mid x)$ is differentiable, 
there exists a sequence of classifiers $\{\hat p_n\}_{n = 1}^{\infty}$, such that almost surely over $x \sim p$, 
\[
  \KL (p(\cdot \mid x) \parallel \hat p_n(\cdot \mid x)) \to 0 \qquad \mbox{as }n \to \infty, 
\]
while  
\[
  \liminf_{n \to \infty} \E_{x \sim p(\cdot \mid y)}\bigl[\|
         \nabla_x \log p(y \mid x) - \nabla_x \log \hat p_n(y \mid x) \|_2^2\bigr] > 0. 
\]
In fact, we can choose $\{\hat p_n\}_{n = 1}^{\infty}$, such that either
\begin{enumerate}
    \item $\sup_{x \in \mathbb{R}^d}\KL(p(\cdot \mid x) \parallel \hat p_n(\cdot \mid x)) = O(1 / n)$ and $\E_{x \sim p(\cdot \mid y)}[\|
         \nabla_x \log p(y \mid x) - \nabla_x \log \hat p_n(y \mid x) \|_2^2 ] = \Omega (1)$, or
    \item $\sup_{x \in \mathbb{R}^d}\KL(p(\cdot \mid x) \parallel \hat p_n(\cdot \mid x)) = O(1 / \sqrt{n})$ and $\E_{x \sim p(\cdot \mid y)}[\|
         \nabla_x \log p(y \mid x) - \nabla_x \log \hat p_n(y \mid x) \|_2^2 ] = \Omega (n)$. 
\end{enumerate}


\end{theorem}

\begin{proof}[Proof of \cref{thm:kl-not-score}]

We prove \cref{thm:kl-not-score} in Appendix \ref{sec:proof:thm:kl-not-score}. A construction of $\{\hat p_n\}_{n = 1}^{\infty}$ can also be found therein. 
\end{proof}

\begin{remark}
    The conditions required for Theorem~\ref{thm:kl-not-score} are quite general: they only require differentiability and that we cannot identify the target label $y$ with probability one. 
    Note that if $p(\{x: p(y \mid x) = 0 \mbox{ or }1\} \mid y) = 1$, 
    then it is impossible to construct such a sequence of classifiers $\{\hat p_n\}_{n = 1}^{\infty}$, since any modification of $p(y \mid x)$ on a set of positive $p(\cdot \mid y)$-probability would incur infinite conditional KL divergence. 
\end{remark}

\cref{thm:kl-not-score} shows that, for very general target distributions, controlling the conditional KL does not guarantee control of the guidance-term MSE. 
Any positive result must therefore rely on additional regularity assumptions on the classifier used.
Such structural regularity arises naturally in diffusion models, since the OU forward process \eqref{eq:ou-forward} smooths the data distribution at all positive times $t>0$, which in turn induces smoothness in the corresponding conditional label probabilities. 
A more detailed discussion is provided in the next section. \textcolor{black}{Our construction is based on introducing high-frequency oscillations and follows standard techniques.}



\subsection{Effective diffusion guidance under classifier smoothness assumptions}

In this section, we show that in diffusion model settings where the intermediate distributions are naturally smooth, 
controlling the conditional KL divergence guarantees control of the guidance  MSE, provided that the classifiers used for guidance satisfy standard smoothness conditions that hold for the true conditional label probabilities. 

\vspace{0.5em}
\begin{theorem}
\label{theorem: good-classifier-preserve-guidance}
Suppose we wish to sample from the conditional distribution $P_{\rm data}(\cdot \mid y)$ associated with a label $y$. 
Fix $T > 0$ and a reverse-time grid $0=t_0<t_1<\cdots<t_N = T - \delta$. Let $\tau_k=t_{k+1}-t_k$ and $s_k = T - t_k$ for $k \in \{0\} \cup [N - 1]$. 
Assume the following:
\begin{enumerate}
\item The data distribution $P_{\mathrm{data}}$ is supported on a bounded set $\mathcal{K} \subseteq \mathbb{R}^{d}$ with $R = \sup_{x \in \mathcal{K}} \norm{x}_{2} < \infty$, in the sense that $P_{\mathrm{data}}(x \in \mathcal{K}) = 1$. 
\item For $k \in \{0\} \cup [N - 1]$, the predicted probability $x \mapsto \hat{p}_{s_k}(y\mid x)$ belongs to $C^2(\R^d)$. In addition,
    \begin{align*}
        \|\nabla_x \log \hat{p}_{s_k}(y\mid x) \|_2 \leq 2 \lambda_{s_k} \sigma_{s_k}^{-2} R, \qquad  \big|\mathrm{Tr}(\nabla_x^2 \log \hat{p}_{s_k}(y\mid x)) \big| \leq 2 \lambda_{s_k}^2 \sigma_{s_k}^{-4} R^2. 
    \end{align*}
    \item There exists a universal constant $C > 0$, such that for all $k \in \{0\} \cup [N - 1]$ and sufficiently small $\epsilon$, 
    \begin{enumerate}[label=(\alph*), ref=\theproposition(\alph*)]
            \item $\mathbb{P}_{x \sim p_{s_k}(\cdot \mid y)} \left( \norm{\nabla_{x}\log \hat p_{s_k}(y \mid x)}_{2} \leq \frac{C \sqrt{d + \log(1/\epsilon)}}{\sigma_{t}} \right) \geq 1 - \epsilon / P_{\rm data}(y)$,  
        \item $\mathbb{P}_{x \sim p_{s_k}(\cdot \mid y)} \left(\, \Big|\,\mathrm{Tr}\big( \nabla_x^2 \log \hat p_{ s_k}(y \mid x) \big) \Big| \leq \frac{C(d + \log(1/\epsilon))}{\sigma_{t}^2} \right) \geq 1 - \epsilon / P_{\rm data}(y).$ 
    \end{enumerate}
    \item The target label satisfies $P_{\rm data}(y) > 0$. 
    \item  For all $k \in \{0\} \cup [N - 1]$,
  \[
    \E_{x \sim p_{s_k}}\Big[
      \KL\big(p_{s_k}(\cdot\mid x)\,\Vert\,\hat p_{s_k}(\cdot\mid x)\big)
    \Big]
    \;\le\; \varepsilon_{s_k}^2 , \qquad \varepsilon_{s_k} \in [0, P_{\rm data}(y) / \sqrt{2}].  \]
\end{enumerate}

Then for all $k \in \{0\} \cup [N - 1]$, there exists a numerical constant $C_0 > 0$, such that 
\begin{align*}
    & \E_{x\sim p_{s_k}(\cdot\mid y)}
  \Big[
    \big\|
      \nabla_x \log p_{s_k}(y\mid x)
      -
      \nabla_x \log \hat p_{s_k}(y\mid x)
    \big\|_2^2
  \Big] \\
  & \leq \frac{C_0\,\varepsilon_{s_k}}{\sigma_{s_k}^2 P_{\rm data}(y)} \Big( d + \log \Big( \frac{R^2 + d}{\sigma_{s_k}^2} \Big) + \log \Big( \frac{1}{\varepsilon_{s_k}} \Big) + \log \Big( \frac{1}{P_{\rm data}(y)}  \Big) \Big).
\end{align*}


\end{theorem}

\begin{remark}
    Ignoring logarithmic factors in $(R, d, 1 / \delta, 1 / P_{\rm data}(y))$, and additionally assume $d \gtrsim \log (1 / \varepsilon_{s_k})$, we have
\[
    \E_{x\sim p_{s_k}(\cdot\mid y)}
  \Big[
    \big\|
      \nabla_x \log p_{s_k}(y\mid x)
      -
      \nabla_x \log \hat p_{s_k}(y\mid x)
    \big\|_2^2
  \Big]
  = O \left( \frac{d\varepsilon_{s_k}}{\sigma_{s_k}^2 } \right). 
\]
\end{remark}

\begin{remark}
By Lemma \ref{propn:tweedie-classifier-conditional} and Corollary \ref{corollary:high-prob-bounds-conditional} in the appendix, the second and the third assumptions of the lemma hold if we replace $\hat p_{s_k}$ with the true probability $p_{s_k}$. 
\textcolor{black}{The first assumption holds for many datasets of practical interest to diffusion models, such as image datasets with bounded pixel values. The $C^2$ smoothness assumption is standard and can be verified for common classifiers such as Gaussian mixture models and logistic regression, as well as for neural networks via backtracking-type methods. We note that both the boundedness and $C^2$ assumptions are standard and have been widely adopted in the diffusion model literature \cite{lee2023convergence, huang2025convergence}. 
The high-probability bounds in the third assumption are weaker than their pointwise counterparts} \textcolor{black}{and hold for simple parametric families.  We however note that the corresponding sub-Gaussian and sub-exponential conditions may not hold for typical deep neural network classifiers.}
The fourth assumption requires that the target label occurs with positive probability.
The final assumption imposes an upper bound on the classification error as measured by the conditional KL.
\end{remark}

\begin{proof}[Proof of \cref{theorem: good-classifier-preserve-guidance}]

We prove \cref{theorem: good-classifier-preserve-guidance} in Appendix \ref{sec:proof:theorem: good-classifier-preserve-guidance}.    
\end{proof}


\textcolor{black}{The bound in Theorem~\ref{theorem: good-classifier-preserve-guidance} is related to, but distinct from, a reverse log-Sobolev inequality: Our goal is to control the guidance error, a quantity over the $x$-space conditioned on a single $y$, using a conditional KL over the label $y$-space without an absolute value operator.} 
\textcolor{black}{We note that establishing a reverse log-Sobolev inequality is usually highly non-trivial and density smoothness alone is generally insufficient to establish such inequalities. 
In particular, it does not follow directly from standard Gagliardo–Nirenberg–type inequalities. 
}

In the next theorem, we assume a standard time-decaying step size for the diffusion model, as in \cite{benton2024nearly}, and under this choice derive a bound on the average MSE of the guidance term.
\textcolor{black}{With \cref{theorem: good-classifier-preserve-guidance}, the following lemma follows directly from the proof in \cite{benton2024nearly}. }

\vspace{1em}

\begin{lemma}
\label{lem:eps_guide_rate}
Fix $T > 0$ and a reverse-time grid $0=t_0<t_1<\cdots<t_N = T - \delta$. Let $\tau_k=t_{k+1}-t_k$ and $s_k = T - t_k$ for $k \in \{0\} \cup [N - 1]$. 
Assume there exists $\kappa>0$ such that $ \tau_k \le \kappa \min\{1,\,s_{k+1}\}$ for $k=0,\dots,N-1.$ 
Define the average guidance term MSE \[ \varepsilon_{\mathrm{guide}}^{2} \;=\; \sum_{k=0}^{N-1}\tau_k\; \mathbb{E}_{x\sim p_{s_k}(\cdot \mid y)}\!\Big[ \big\|\nabla_{x} \log p_{s_k}(y \mid x) - \nabla_{x} \log \hat{p}_{s_k}(y \mid x)\big\|_{2}^{2} \Big]. \] 
Under the assumptions of \Cref{theorem: good-classifier-preserve-guidance}, it holds that 
\begin{align*}
    \varepsilon_{\mathrm{guide}}^{2} \leq c_0 \sum_{k = 0}^{N - 1}\frac{\kappa  \varepsilon_{s_k}}{ P_{\rm data}(y)} \Big( d + \log \Big( \frac{R^2 + d}{\delta} \Big) + \log \Big( \frac{1}{\varepsilon_{s_k}} \Big) + \log \Big( \frac{1}{P_{\rm data}(y)}  \Big) \Big), 
\end{align*}
where $c_0 > 0$ is a positive numerical constant. 
\end{lemma}

\begin{proof}[Proof of Lemma \ref{lem:eps_guide_rate}]
We prove Lemma \ref{lem:eps_guide_rate} in Appendix \ref{sec:proof:lem:eps_guide_rate}. 
\end{proof}

The average guidance MSE $\varepsilon_{\mathrm{guide}}^{2}$ will later be used to control the sampling error in Section~\ref{sec:sampling}, and $\kappa$ can be interpreted as controlling the maximum step size.

Theorem~\ref{thm:kl-not-score} and Theorem~\ref{theorem: good-classifier-preserve-guidance} together provide a complete characterization of classifier guidance. 
On the one hand, in the absence of classifier smoothness, control of the conditional KL divergence does not guarantee control of the guidance term MSE. 
On the other hand, for bounded target distributions, if the classifier satisfies the same smoothness conditions as the ground-truth conditional label probabilities, 
this pathology disappears, and small conditional KL does imply control of the guidance MSE.

\subsection{Rate optimality}

The bound in Theorem~\ref{theorem: good-classifier-preserve-guidance} shows that as the conditional KL divergence vanishes at rate $O(\varepsilon^2)$, the guidance term MSE decays at order $O(\varepsilon)$ (ignoring dependence on other problem parameters). 
A natural question is whether this rate can be improved.


Our next result shows that such an improvement is impossible. 
Specifically, we construct a data distribution and a family of classifiers that satisfy Assumptions~(1)–(5) of Theorem~\ref{theorem: good-classifier-preserve-guidance}, such that the conditional KL divergence scales as $O(\varepsilon^2)$, while the guidance MSE is $\Omega(\varepsilon)$. 
Our construction demonstrates that the $\varepsilon$-dependence in Theorem \ref{theorem: good-classifier-preserve-guidance} is tight.

\vspace{1em}

\begin{theorem}
\label{prop:sharpness}
There exists a data distribution $P_{\mathrm{data}}$, 
such that for any sufficiently small $\varepsilon$, there exists a sequence of probabilistic classifiers $x \mapsto \hat p_{t}(\cdot \mid x)$ for $t \in \{s_k: k \in \{0\} \cup [N - 1]\}$ that satisfy Assumptions~\textnormal{(1)--(5)} of Theorem~\textnormal{\ref{theorem: good-classifier-preserve-guidance}}, such that 
\begin{enumerate}
\item[(i)] The classifiers achieve small conditional KL:
\[
  \sup_{t \in \{s_k: k \in \{0\} \cup [N - 1]\}}\E_{X_t \sim p_t(\cdot)}\Big[
    \KL\big(p_t(\cdot\mid X_t)\,\Vert\,\hat p_{t}(\cdot\mid X_t)\big)
  \Big]
  \;\le\; C_1\,\varepsilon^2
\]
for a constant $C_1 > 0$ that is independent of $\varepsilon$.

\item[(ii)] Guidance term MSEs are bounded below: 
\[
  \inf_{t \in \{s_k: k \in \{0\} \cup [N - 1]\}}\E_{X_{t}\sim p_{t}(\cdot\mid y)}
  \Big[
    \big\|
      \nabla_x \log p_{t}(y\mid X_{t})
      -
      \nabla_x \log \hat p_{t}(y\mid X_{t})
    \big\|_2^2
  \Big]
  \;\ge\;
  C_2\,\varepsilon
\]
for a constant $C_2>0$ that is independent of $\varepsilon$.
\end{enumerate}

\end{theorem}

\begin{proof}[Proof of Theorem \ref{prop:sharpness}]
    We prove Theorem \ref{prop:sharpness} in Appendix \ref{sec:proof:prop:sharpness}. We also present our construction therein. 
\end{proof}

\textcolor{black}{In summary, we not only provide an upper bound on the $L^2$ guidance error in terms of conditional KL, but also derive a nearly optimal rate. 
To our knowledge, such rate is not a trivial consequence of existing results, which we believe is a meaningful contribution.}

\subsection{Implications for conditional sampling error}
\label{sec:sampling}

In practice, diffusion models are implemented by discretizing continuous-time processes and using learned score functions. 
Consequently, the sampling error depends not only on the guidance error, but also on the discretization error and the score estimation error. 
For completeness, we analyze time-discretized guided diffusion in this section. 
Our proof largely follows that of \cite{benton2024nearly}.

We consider DDPM-style samplers equipped with an exponential integrator.
Specifically, we fix a time grid $0=t_0<t_1<\cdots<t_N=T-\delta$, and define $\tau_k=t_{k+1}-t_k$ and $s_k=T-t_k$ for $k \in \{0\} \cup [N - 1]$.
For each step $k = 1, 2, \cdots, N$, we train a classifier $\hat p_{s_k}$ and a score estimate $s_{\theta}(\cdot, s_k)$. 
Given these components, a DDPM-type sampler starts from $\hat X_0^{(y)} \sim N(0, I_d)$ and, at the $k$-th update step for $k = 1, 2, \cdots, N$, set
\begin{align}
\label{eq:guided-DDPM}
    \hat X_{t_{k + 1}}^{(y)} = e^{\tau_k}\hat X_{t_k}^{(y)} + 2(e^{\tau_k} - 1) \big[ s_{\theta}(\hat X_{t_k}^{(y)}, s_k) + \gamma \,\nabla\log \hat p_{s_k}(y \mid X_{t_k}^{(y)}) \big] + \sqrt{e^{2 \tau_k} - 1}\, Z_k, 
\end{align}
where $Z_k \sim_{i.i.d.} N(0, I_d)$ are independent of the previous estimates, $\gamma > 0$ is a hyperparameter that controls the guidance strength. \textcolor{black}{
For $k=0,1,\ldots,N$, we denote by
$\hat q_{\gamma,t_k}$ the marginal distribution of
$\hat X_{t_k}^{(y)}$, so that $\hat q_{\gamma,T-\delta}$ denotes the output distribution of sampler \eqref{eq:guided-DDPM}. We compare $\hat q_{\gamma,T-\delta}$ against the corresponding ideal guided marginal
distribution of $\overleftarrow{X}_t^{(y)}$. Specifically, for $t\in[0,T]$, define the $\gamma$-tilted distribution
\begin{equation}
p_{t,\gamma}(x)
=
\frac{
    p_t(x)\,p_t(y\mid x)^\gamma
}{
    Z_{t,\gamma}
},
\;\;
Z_{t,\gamma}
=
\E_{X\sim p_t}
\left[
    p_t(y\mid X)^\gamma
\right].
\label{eq:gamma-tilted-dist}
\end{equation}
We then consider the ideal classifier-guided reverse process
\begin{equation}
    \dd \overleftarrow{X}_t^{(y)}
    =
    \Big\{
        \overleftarrow{X}_t^{(y)}
        + 2\nabla_x\log p_{T-t}(\overleftarrow{X}_t^{(y)})
        + 2\gamma\nabla_x\log p_{T-t}
        (y\mid \overleftarrow{X}_t^{(y)})
    \Big\}\dd t
    + \sqrt{2}\,\dd B_t',
    \;\;
    \overleftarrow{X}_0^{(y)}\sim p_{T,\gamma}.
    \label{eq:ideal-guided-sde}
\end{equation}
Let $q_{\gamma,t}$ denote the marginal distribution of
$\overleftarrow{X}_t^{(y)}$. Our analysis below bounds the sampling error between the output distribution $\hat q_{\gamma,T-\delta}$ and the corresponding ideal guided marginal $q_{\gamma,T-\delta}$ for $\gamma>1$.
}


To state our result, we impose the following assumption on the score estimation error. 

\vspace{1em}

\begin{assumption}
\label{assumption:score}
    The score estimate $s_{\theta}$ satisfies
    \[
        \sum_{k=0}^{N-1} \tau_k \mathbb{E}_{X \sim p_{T - t_{k}}} [ \norm{\nabla_{x} \log p_{T - t_{k}}(X) -s_{\theta}(X, T - t_{k})}_{2}^{2}] \leq \varepsilon_{\mathrm{score}}^2. 
    \]
\end{assumption}

Under Assumption~\ref{assumption:score} and the assumptions of \cref{theorem: good-classifier-preserve-guidance}, we establish the following upper bound on the sampling error.

\textcolor{black}{
\begin{theorem}
\label{thm:discretization}
Suppose $T\geq 1$ and assume there exists $\kappa>0$ such that for each
$k=0,\dots,N-1$,
$\tau_k \leq \kappa \min\{1,\,T-t_{k+1}\}
=\kappa\min\{1,\,s_{k+1}\}$.
Then, under Assumption \ref{assumption:score} and the assumptions of
\cref{theorem: good-classifier-preserve-guidance}, it holds that
\begin{align*}
\label{eq:thm_disc_main}
D_{\mathrm{KL}}\big(q_{\gamma,T-\delta}\,\|\,\hat q_{\gamma,T-\delta}\big)
\;\leq\;
C_{\rm s}\Bigg(
\frac{
\varepsilon_{\mathrm{score}}^2
+\gamma^2\varepsilon_{\mathrm{guide}}^2
+(2\gamma^2-2\gamma+1)(d+R^2)(\kappa T+\kappa^2N)
}{
P_{\rm data}(y)^\gamma
}
+
\frac{\gamma^2(d+R^2)e^{-2T}}
{P_{\rm data}(y)^{2\gamma}}
\Bigg).
\end{align*}
where we recall $\varepsilon_{\mathrm{guide}}^2$ is from
\Cref{lem:eps_guide_rate}, and $C_{\rm s}$ is a positive numerical constant.
\end{theorem}
}
\begin{proof}[Proof of Theorem \ref{thm:discretization}]
    We prove Theorem \ref{thm:discretization} in Appendix \ref{sec:proof:thm:discretization}. 
\end{proof}



\begin{remark}
    Note that the upper bound on $\varepsilon_{\mathrm{guide}}^{2}$ in Theorem \ref{theorem: good-classifier-preserve-guidance} scales linearly with the dimension $d$, and this dependence cannot be improved for the KL divergence between $d$-dimensional distributions (e.g., for product distributions). 
\end{remark}

\section{Experimental results}
\label{sec:experiments}
In this section, we present numerical experiments that support our theoretical findings.

\subsection{Counterexample: controlling conditional KL does not guarantee effective guidance}

\textcolor{black}{We numerically validate Theorem~\ref{thm:kl-not-score} by constructing a counterexample in a binary classification setting. We use our construction to demonstrate that a classifier can achieve vanishingly small conditional KL, while exhibiting substantial guidance term misalignment, thereby demonstrating that controlling conditional KL does not necessarily guarantee effective guidance.}

\paragraph{Experimental setup.}

\textcolor{black}{We consider a binary classification problem with label space $\mathcal{Y}=\{0,1\}$ and a $50$-dimensional input space. We denote by $p$ the target distribution, and assume the guidance is towards the conditional distribution $p(\cdot \mid y = 1)$. We set $p(x)=\mathcal{N}(0,I_{50})$. Conditioned on $x$, we take
$
p(y=1 \mid x) = 0.5 + 0.3 \tanh(x_1),
$
which yields probabilities in the range $[0.2, 0.8]$. We set $\gamma = 0.3$ and define the set
$
A_\gamma = \{x : p(y=1 \mid x) < 1-\gamma\}.
$
Following the construction in the proof of Theorem~\ref{thm:kl-not-score} (Appendix~\ref{sec:proof:thm:kl-not-score}), we construct the following probabilistic classifier by perturbing the true conditional label probabilities, with the perturbation indexed by $n$:
\[
\hat{p}_n(y=1 \mid x) = \begin{cases}
p(y=1 \mid x) \cdot [1 + \delta_n \sin(n x_1)], & \text{if } x \in A_\gamma, \\
p(y=1 \mid x), & \text{if } x \notin A_\gamma,
\end{cases}
\]
with $\hat{p}_n(y=0 \mid x) = 1 - \hat{p}_n(y=1 \mid x)$ to ensure normalization. The perturbation acts only on the first coordinate $x_1$, so the remaining $49$ coordinates do not contribute to the guidance error. The parameter $n \in \{10, 30, \ldots, 1000\}$ determines the oscillation frequency of the perturbation, and the amplitude $\delta_n$ also scales with $n$. For the Monte Carlo estimation of the conditional KL divergence and the guidance term MSE, we utilize $N = 10{,}000$ independent sample points. We examine the two regimes discussed in the proof of Theorem~\ref{thm:kl-not-score}(Appendix~\ref{sec:proof:thm:kl-not-score}): (1) In Regime 1 we set $\delta_n = 1/n$, and the theorem predicts that the conditional KL decays at rate $O(1/n)$, while the guidance MSE is $\Omega(1)$; (2) In Regime 2 we set $\delta_n = 1/\sqrt{n}$. In this case, the theorem shows that the conditional KL decays at a slower rate $O(1/\sqrt{n})$, and the guidance MSE diverges to infinity at rate $\Omega(n)$.}

\paragraph{Simulation outcomes.}

\textcolor{black}{Figure~\ref{fig:theorem31-convergence} presents the simulation outcomes for both regimes.
We measure two quantities: the conditional KL divergence
$
\mathbb{E}_{x \sim p}\left[\mathrm{KL}(p(\cdot \mid x) \,\|\, \hat{p}_n(\cdot \mid x))\right],
$
and the guidance term MSE \newline
$\mathbb{E}_{x \sim p(\cdot \mid y = 1)}\left[\|\nabla_x \log p(y=1 \mid x) - \nabla_x \log \hat{p}_n(y=1 \mid x)\|_2^2\right].
$
We discuss the results from these two regimes separately in the following.}

\begin{figure}[h!]
\centering
\includegraphics[width=\textwidth]{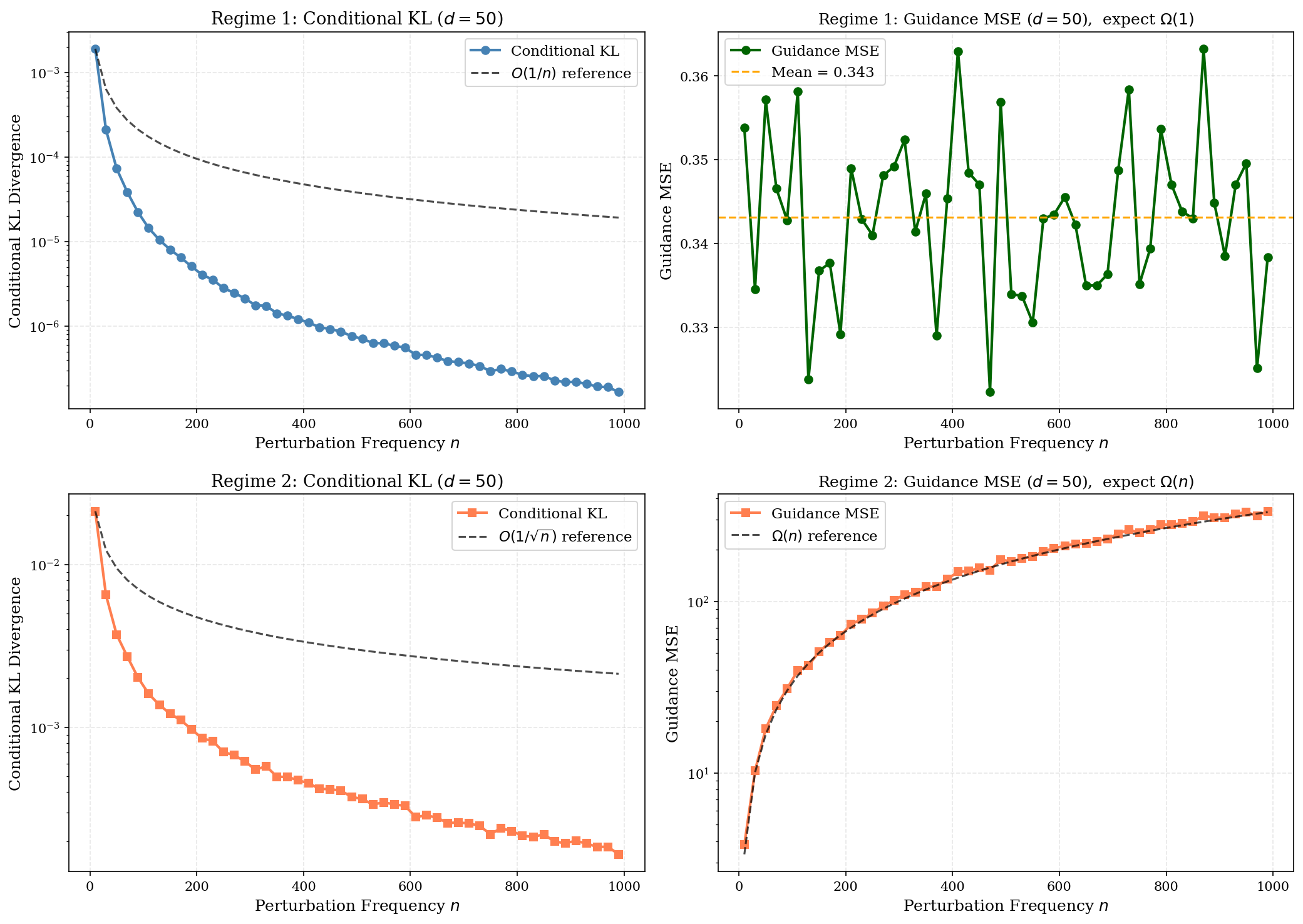}
\caption{\textcolor{black}{Numerical validation of Theorem~\ref{thm:kl-not-score}. 
The plots in the top and bottom panels correspond to Regimes 1 and 2, respectively. 
The first and third plots display the conditional KL divergence, and the second and fourth plots show the guidance MSE. }}
\label{fig:theorem31-convergence}
\end{figure}

\textcolor{black}{\begin{itemize}
    \item \textbf{Regime 1} ($\delta_n = 1/n$): As shown in the first plot of  Figure~\ref{fig:theorem31-convergence}, the conditional KL (top-left) decreases from approximately $2 \times 10^{-3}$ at $n=10$ to approximately $2 \times 10^{-7}$ at $n=1000$, decaying faster than the $O(1/n)$ upper bound. On the other hand, the guidance MSE (second plot of Figure~\ref{fig:theorem31-convergence}) remains nearly constant at around $0.343$ across all frequencies, confirming the $\Omega(1)$ behavior predicted by the theorem.  
    \item \textbf{Regime 2} ($\delta_n = 1/\sqrt{n}$): As shown in the third plot of Figure~\ref{fig:theorem31-convergence}, the conditional KL divergence (bottom-left) decreases from approximately $2 \times 10^{-2}$ at $n=10$ to approximately $2 \times 10^{-3}$ at $n=1000$, also decaying to zero faster than the upper bound $O(1/\sqrt{n})$. In contrast, the guidance error (fourth plot of Figure~\ref{fig:theorem31-convergence}) increases dramatically from approximately $4$ at $n=10$ to approximately $300$ at $n=1000$, exhibiting clear $\Omega(n)$ growth. 
\end{itemize}}

\paragraph{Effect of classifier smoothness.}
\textcolor{black}{
To further investigate the role of classifier smoothness in guidance accuracy, 
we extend the preceding counterexample experiment by varying the perturbation 
amplitude $\delta_n$. Recall that the perturbation in our construction is of the form
$\delta_n \sin(n x_1)$, whose first derivative has magnitude proportional to
$n\delta_n$. We consider four choices
\[
\delta_n \in 
\left\{
\frac{1}{n^2},\,
\frac{1}{n},\,
\frac{1}{\sqrt{n}},\,
\frac{1}{n^{1/4}}
\right\}.
\]
These choices correspond to progressively less regular perturbations as $n$ increases:
$n\delta_n = 1/n$ vanishes for $\delta_n=1/n^2$, remains bounded for
$\delta_n=1/n$, and grows for $\delta_n=1/\sqrt{n}$ and
$\delta_n=1/n^{1/4}$.
}

\textcolor{black}{
To isolate the effect of this regularity from classifier accuracy, we compare
the four perturbation regimes at approximately matched levels of conditional
KL divergence. Since the conditional KL values attained at the discrete
frequencies do not coincide exactly across the four regimes, we evaluate the
curves on a shared conditional-KL grid using interpolation. Figure~\ref{fig:common-kl-smoothness}
plots the resulting guidance MSE against conditional KL.
}

\begin{figure}[htbp!]
    \centering
    \includegraphics[width=0.58\textwidth]{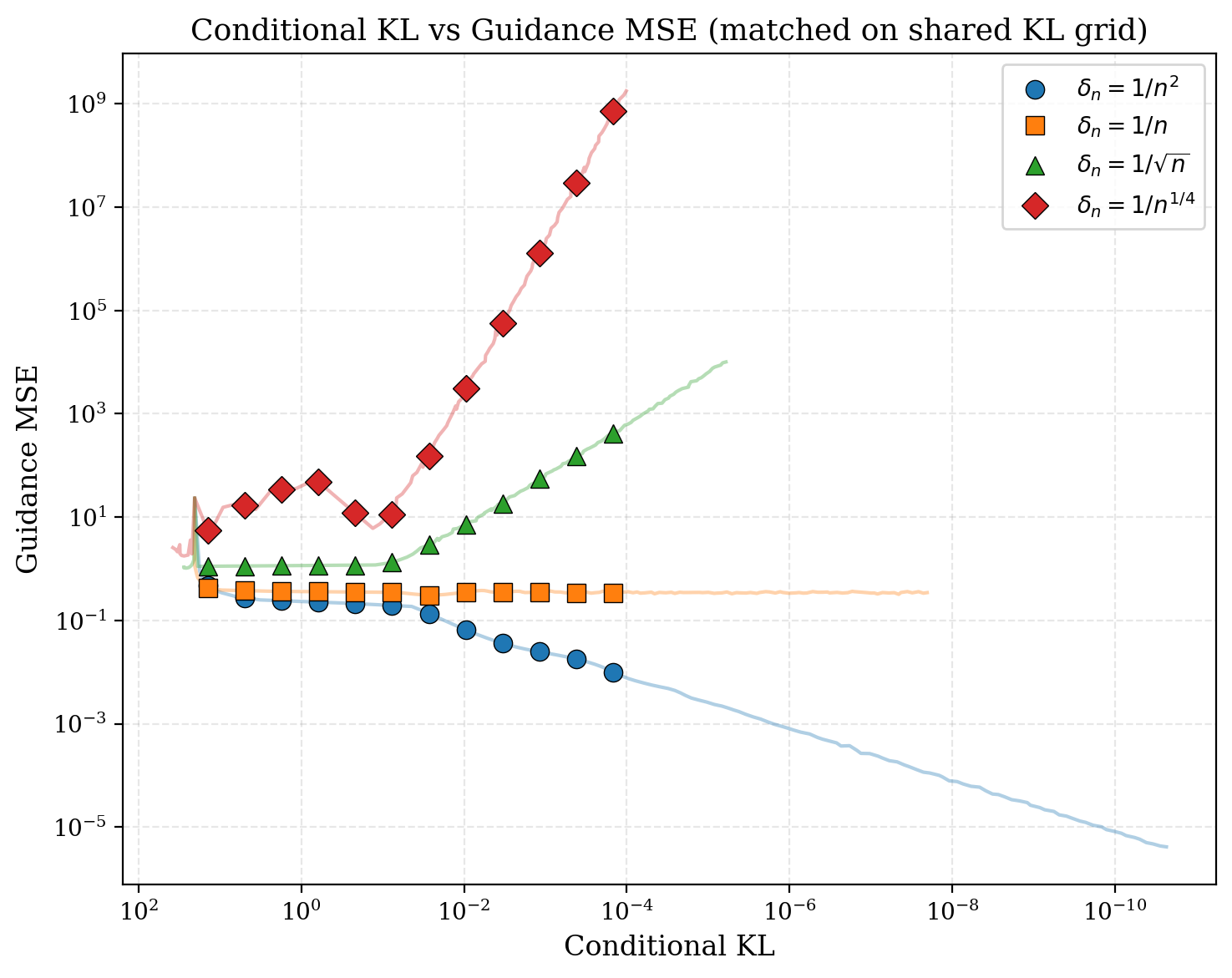}
    \caption{\textcolor{black}{
    Effect of perturbation regularity on guidance error at comparable levels
    of conditional KL divergence. We consider
    $\delta_n \in \{1/n^2,\,1/n,\,1/\sqrt{n},\,1/n^{1/4}\}$ and evaluate the
    corresponding guidance MSEs on a shared conditional-KL grid.
    }}
    \label{fig:common-kl-smoothness}
\end{figure}

\textcolor{black}{
Figure~\ref{fig:common-kl-smoothness} demonstrates that classifier accuracy
as measured by conditional KL alone does not determine the quality of the
guidance vector. At comparable conditional KL levels, the guidance MSE differs
substantially across the four perturbation regimes. In particular, the
$\delta_n=1/n^2$ regime, for which the perturbation derivative vanishes as
$n$ increases, yields the smallest guidance error, whereas the less regular
choices $\delta_n=1/\sqrt{n}$ and $\delta_n=1/n^{1/4}$ produce substantially
larger guidance errors. The intermediate choice $\delta_n=1/n$ yields a
guidance error that remains approximately constant. These results empirically
illustrate that increasing classifier regularity suppresses the high-frequency
gradient errors underlying the counterexample in
Theorem~\ref{thm:kl-not-score}, and provide complementary motivation for the
smoothness conditions imposed in
Theorem~\ref{theorem: good-classifier-preserve-guidance}.
}

\subsection{Guidance error control with classifier smoothness}
\label{sec:regularity}

\textcolor{black}{We now validate Theorem~\ref{theorem: good-classifier-preserve-guidance} in the setting of a Gaussian mixture model (GMM). We consider a binary label space $\mathcal{Y} = \{0,1\}$ and a sample space $\mathbb{R}^{50}$. We denote the target distribution by $p$. We set $p(y = 0) = p(y = 1) = 0.5$, $\mu_0 = [-2, 0, \ldots, 0]^\top \in \mathbb{R}^{50}$, $\mu_1 = [2, 0, \ldots, 0]^\top \in \mathbb{R}^{50}$, and assume the following GMM:
\begin{align*}
    p(x \mid y = 0) = \mathcal{N}(\mu_{0},\, 0.5\,I_{50}), \qquad 
    p(x \mid y = 1) = \mathcal{N}(\mu_1,\, 0.5\,I_{50}).
\end{align*}
Consider the forward noising process of the diffusion model in eq. 1 initialized at $p$, and denote by $p_t$ the distribution at time $t$.
One can verify that $p_t(y=0) = p_t(y=1) = 0.5$,
\begin{align*}
    p_t(x \mid y = 0) = \mathcal{N}\!\left(\lambda_t\mu_{0},\;(0.5\lambda_t^2 + \sigma_t^2)\,I_{50}\right),
    \qquad
    p_t(x \mid y = 1) = \mathcal{N}\!\left(\lambda_t\mu_1,\;(0.5\lambda_t^2 + \sigma_t^2)\,I_{50}\right).
\end{align*}
The true conditional probabilities $p_t(y \mid x)$ can be computed analytically via Bayes' rule.
For diffusion time $t  = 1.0$, we generate independent training samples $\{(x_i, y_i)\}_{i=1}^N$ from $p$ and fit a two-component Gaussian mixture model $\hat{p}_t$ using the EM algorithm with a diagonal covariance structure, which yields smooth classifiers and remains statistically reliable in high dimensions.
We consider sample sizes $N \in \{200, 500, 1000, 2000, 3000, 5000\}$; sample sizes below $200$ are excluded as diagonal GMM estimation becomes unreliable in $\mathbb{R}^{50}$ at very small $N$.
Each experiment is repeated over $20$ independent trials, and we report the mean and standard error.
For each $(t, N)$ setting, we use Monte Carlo simulation with a test set of size $1{,}000$ to calculate the conditional KL divergence and the guidance MSE.}

\paragraph{Simulation outcomes.}

\begin{figure}[htbp!]
    \centering
    \includegraphics[width=0.9\textwidth]{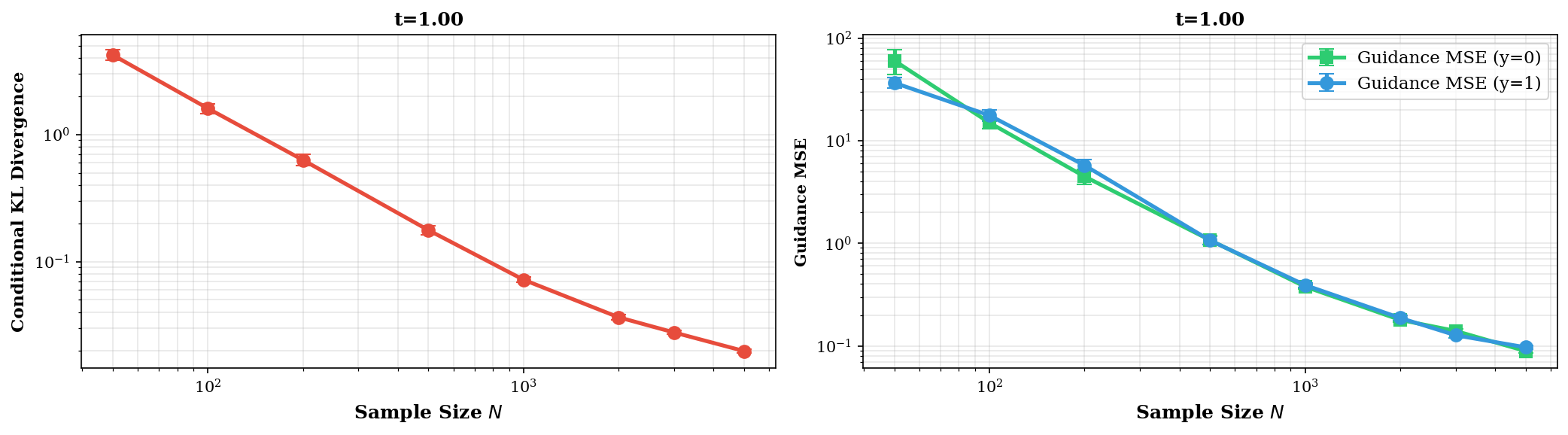}
    \caption{\textcolor{black}{Numerical validation of Theorem~ 3 
    under a GMM with $d=50$. The figures correspond to a diffusion time $t = 1.0$. The left panel displays the conditional KL divergence 
    and the right panel shows the guidance MSE for $y=0$ and $y=1$.}}
    \label{fig:gmm-d50}
\end{figure}
\textcolor{black}{Figure~\ref{fig:gmm-d50} plots the conditional KL divergence and the guidance MSE as 
functions of the sample size $N$, for each diffusion time $t  = 1.0$.
We observe that both the conditional KL divergence and the guidance MSE decrease 
consistently as $N$ increases across all time points, with similar decay trends. Similar observations are also made for diffusion time points $t \in \{0.01, 0.1, 3.0\}$. 
This indicates that smaller classifier conditional KL indeed leads to smaller guidance 
MSE under classifier smoothness conditions, empirically validating 
Theorem 3 in the $50$-dimensional setting. }

\subsection{Real-world example using MNIST dataset}
\label{subsec:end-to-end-guidance-exp}

\textcolor{black}{Finally, we validate our theoretical findings on the MNIST dataset \cite{mnist}. Specifically, we compare two classifiers trained using neural networks with similar classification accuracy but different smoothness properties. Our experiments show that these classifiers exhibit different levels of effectiveness when used for guidance.
}


\paragraph{Experimental setup.}

\textcolor{black}{We consider the 10-class MNIST digit classification task with label space $\mathcal{Y}=\{0,1,\ldots,9\}$ and sample space $\mathbb{R}^{28\times28}$.
We assume that the images are normalized so that pixel values lie in $[-1, 1]$. 
We take $\delta = 0.01$ and $T = 1.5$. 
We train two time-dependent neural classifiers $\hat{p}_t^{\text{smooth}}(y \mid x)$ and $\hat{p}_t^{\text{non-smooth}}(y \mid x)$ to predict labels from noise-corrupted images. 
Both classifiers are trained by minimizing the cross-entropy objective and use the same overall architecture, including the same convolutional layers, kernel sizes, normalization, and average pooling, differing only in their activation functions. 
In particular, the smooth classifier employs GELU and tanh activations, which have continuous higher-order derivatives, whereas the non-smooth classifier uses ReLU and Hardtanh activations, which introduce transition boundaries at which the gradient is discontinuous. 
Both classifiers are trained for 500 epochs on 60,000 MNIST training samples with a batch size of 256, and they achieve comparable classification accuracies of 0.867 and 0.861 and cross-entropy losses of 0.4013 and 0.4018, respectively. }

\paragraph{Simulation outcomes.}

\begin{figure}[ht]
\centering
\includegraphics[width=0.7\textwidth]{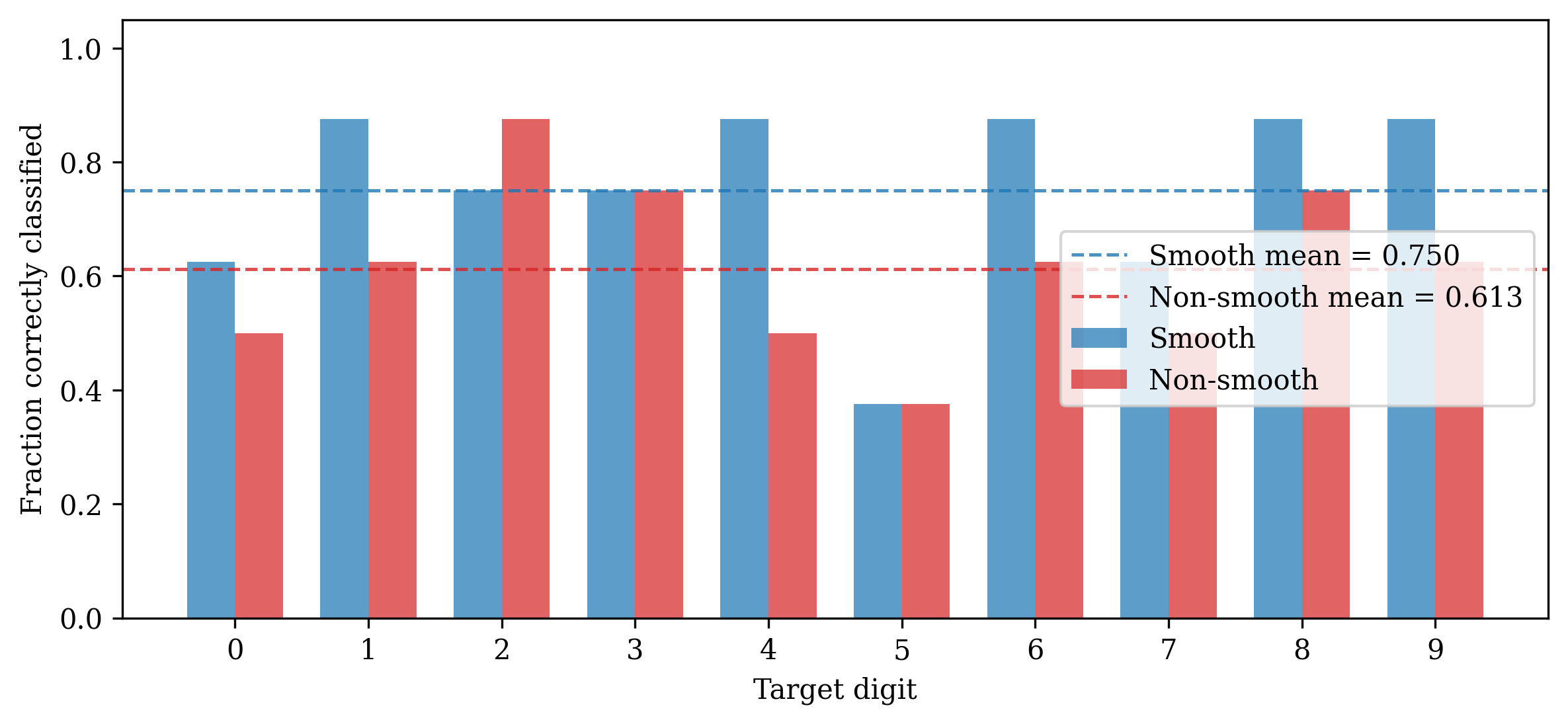}
\vspace{0.5em}
\includegraphics[width=0.7\textwidth]{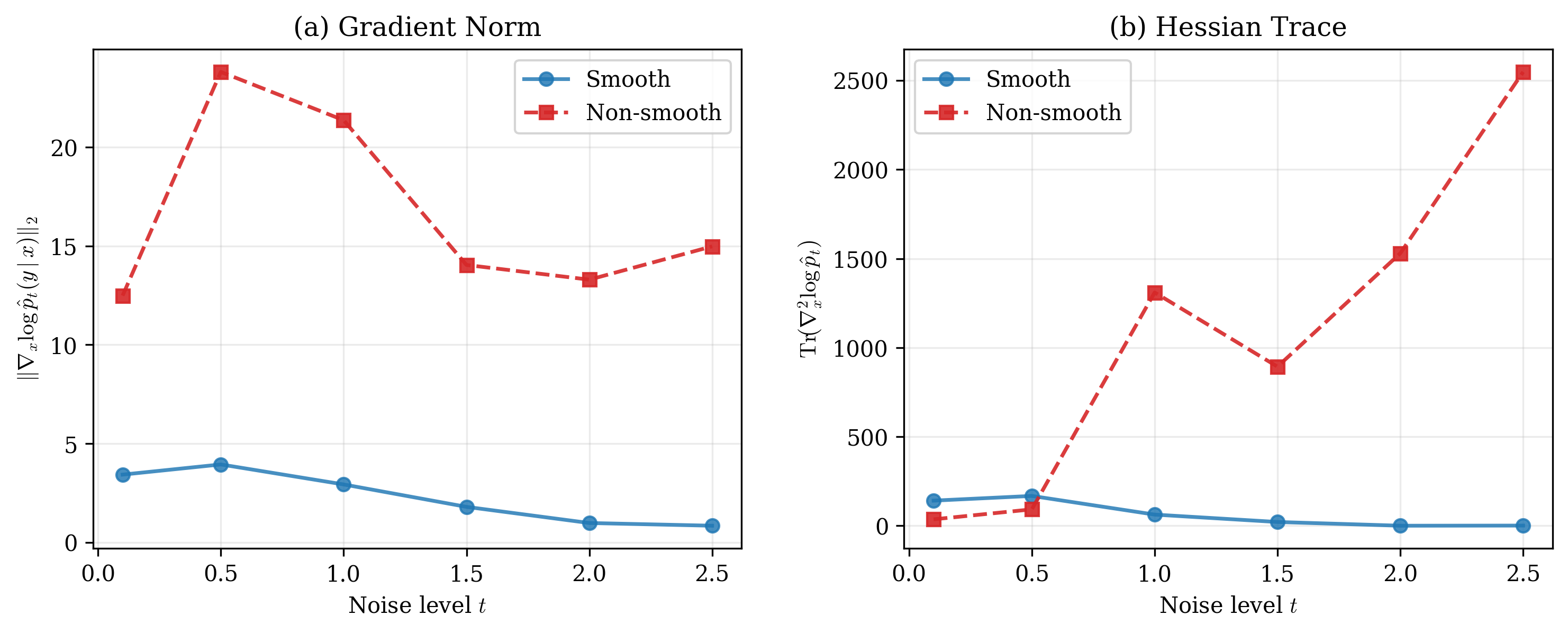}
\caption{\textbf{Top panel:} Guidance effectiveness by target digit. \textbf{Bottom panel:} Averaged gradient norm and Hessian trace as functions of diffusion time $t$.}
\label{fig:test}
\end{figure}


\textcolor{black}{Figure~\ref{fig:test} summarizes the numerical results. The top panel shows the fraction of generated images that are correctly classified by a held-out MNIST classifier trained independently on clean images, while the bottom panel reports the averaged gradient norm and Hessian trace as functions of diffusion time $t$. We observe that the smooth classifier exhibits significantly smaller gradient norms and Hessian traces than the non-smooth classifier. In addition, the smooth classifier achieves stronger guidance performance, as measured by the fraction of generated samples that are correctly classified, supporting our theoretical findings. Further details of these numerical experiments are provided in Appendix~\ref{appendix:end-to-end-guidance-exp}.}

\section{Conclusion}

In this work, we identify tight conditions under which classification error (measured by conditional KL divergence) controls the effectiveness of diffusion guidance.
Our negative result 
shows that low classification error alone is insufficient, by exhibiting a concrete example in which the classifier attains small conditional KL while the resulting guidance terms are severely misaligned.
Complementarily, our positive result 
establishes that, under additional classifier smoothness assumptions, a probabilistic classifier with conditional KL $\varepsilon^2$ induces a guidance MSE of order $\widetilde O(d\varepsilon)$.
We prove such dependency on $\varepsilon$ is optimal, 
and further leverage our result to derive sampling error bounds for classifier-guided diffusion models. 
Several directions merit further investigation, including extending our analysis to non-compactly supported distributions and exploring alternative classifier training procedures. 
By rigorously connecting classifier training error to diffusion guidance quality, this work provides both theoretical foundations and practical insights for classifier-guided diffusion models.

\section*{Acknowledgments}

We thank Edgar Dobriban, Jiaoyang Huang, and Xianli Zeng for inspiring conversations. 

\bibliographystyle{alpha}
\bibliography{refs}

\newpage
\appendix

\section{Proofs of the main theoretical results}

\subsection{Proof of \cref{thm:kl-not-score}}
\label{sec:proof:thm:kl-not-score}

For $\gamma \in (0, 1)$, define $A_{\gamma} = \{x: 0 < p(y \mid x) < 1 - \gamma\}$. By assumption, there exists $\gamma \in (0, 1)$, such that $p(A_{\gamma} \mid y) > 0$. 
Denote by $x_1$ the first coordinate of $x$. 
Let  $\{\delta_n\}_{n = 1}^{\infty}$ be a sequence of amplitudes satisfying $\delta_n \downarrow 0$ as $n \to \infty$. 
For $x \in A_{\gamma}$, define the following perturbation: 
\begin{align*}
	& \hat p_n(y \mid x) = p(y \mid x) \, \big ( 1 + \delta_n \sin(nx_1)  \big), \\
	& \hat p_n(y' \mid x) = p(y' \mid x) \cdot  \frac{1 - p(y \mid x) - \delta_n \sin (n x_1) \, p(y \mid x)}{1 - p(y \mid x)}, \;\; y' \in \mathcal{Y} \backslash \{y\}. 
\end{align*}
For $x \notin A_{\gamma}$, we let $\hat p_n(\cdot \mid x) = p(\cdot \mid x)$. 
For any $\delta_n \in [0, \gamma / (2 - 2\gamma)\,)$, the conditional density $\hat p_n(\cdot \mid x)$ is well defined and adds up to one almost surely with respect to $x \sim p$. 

Note that for $x \notin A_{\gamma}$, by definition $\KL (p(\cdot \mid x) \parallel \hat p_n(\cdot \mid x)) = 0$. 
For $x \in A_{\gamma}$, 
\begin{align*}
	\KL (p(\cdot \mid x) \parallel \hat p_n(\cdot \mid x)) = & - p(y \mid x) \log (1 + \delta_n \sin (n x_1)) -\sum_{y' \in \mathcal{Y} \backslash \{y\}} p(y' \mid x) \log \frac{1 - p(y \mid x) - \delta_n \sin (n x_1) \, p(y \mid x)}{1 - p(y \mid x)} \\
	\leq & \, \delta_n + \frac{\delta_n p(y \mid x)}{1 - p(y \mid x) - \delta_n \sin (n x_1) \, p(y \mid x)} \leq 2\delta_n \gamma^{-1}. 
\end{align*}
where we have applied the standard bounds $\log(1 + x) \leq x$ for $x \geq 0$ and $-\log(1-x) \leq \frac{x}{1-x}$ for $0 \leq x < 1$. Therefore, as $n \to \infty$, almost surely over $x \sim p$
\begin{align*}
	 \KL (p(\cdot \mid x) \parallel \hat p_n(\cdot \mid x)) \leq 2\delta_n \gamma^{-1} \to 0. 
\end{align*}
Next, we show that the guidance vector corresponding to $\hat p_n$ is misaligned with that of $p$. 
Note that for $x \in A_{\gamma}$, 
\begin{align*}
	\nabla_x \log \hat p_n(y \mid x) = & \nabla_x \log p(y \mid x) + \nabla_x \log (1 + \delta_n \sin (n x_1)) \\
	= &  \nabla_x \log p(y \mid x) + \frac{\delta_n n \cos (n x_1)}{1 + \delta_n \sin (n x_1)} e_1, 
\end{align*}
where $e_1 \in \R^d$ denotes the first standard basis vector. Therefore, using the fact that $\cos^2 x = (1 + \cos (2x)) / 2$ and $\delta_n \in [0, \gamma / (2 - 2\gamma)\,)$, we see that
\begin{align*}
	& \E_{x \sim p(\cdot \mid y)} \big[ \| \nabla_x \log p(y \mid x) - \nabla_x \log \hat p_n(y \mid x) \|_2^2 \big]  = \int_{A_{\gamma}} \frac{\delta_n^2 n^2 \cos (n x_1)^2}{(1 + \delta_n \sin (n x_1))^2} p(x \mid y)\, \dd x \\
	& \geq \frac{(1 - \gamma)^2\delta_n^2 n^2}{2} \, \Big[ p(A_{\gamma} \mid y) + \int_{\R} \cos (2n x_1) \Big( \int_{\R^{d - 1}} \mathbbm{1}_{A_{\gamma}}(x)  p(x_1, x_{2:d}\mid y) \dd x_{2:d} \Big) \dd x_1 \Big]. 
\end{align*}
Note that $x_1 \mapsto \int_{\R^{d - 1}} \mathbbm{1}_{A_{\gamma}}(x)  p(x_1, x_{2:d} \mid y) \dd x_{2:d}$ as a function over $\R$ is $L^1$ integrable. 
Therefore, by the Riemann-Lebesgue lemma, it holds that 
\begin{align*}
	\int_{\R} \cos (2n x_1) \Big( \int_{\R^{d - 1}} \mathbbm{1}_{A_{\gamma}}(x)  p(x_1, x_{2:d} \mid y) \dd x_{2:d} \Big) \dd x_1 \to 0 \;\; \mbox{as }n \to \infty, 
\end{align*}
hence for a sufficiently large $n$, one can take $|I_{n}| \leq \frac{1}{2}p(A_{\gamma} \mid y)$ which yields
\begin{align*}
	\E_{x \sim p(\cdot \mid y)} \big[ \| \nabla_x \log p(y \mid x) - \nabla_x \log \hat p_n(y \mid x) \|_2^2 \big] \geq \frac{(1 - \gamma)^2p(A_{\gamma}\mid y)\delta_n^2 n^2}{4}. 
\end{align*}
It then suffices to choose a sequence $\{\delta_n\}_{n = 1}^{\infty}$ such that $\delta_n n \geq c > 0$ while $\delta_n \downarrow 0$ as $n \to \infty$. 
We present below two illustrative scaling regimes. 
\begin{table}[ht]
\label{table:two-regimes}
\centering 
\begin{tabular}{lllll}
\multicolumn{1}{l|}{Choice of $\delta_n$}      & \multicolumn{1}{l|}{Cross entropy loss} & Guidance vector error &  &  \\ \cline{1-3}
\multicolumn{1}{l|}{$\delta_n = 1 / n$}        & \multicolumn{1}{l|}{$O(1 / n)$}       & $\Theta(1)$           &  &  \\ \cline{1-3}
\multicolumn{1}{l|}{$\delta_n = 1 / \sqrt{n}$} & \multicolumn{1}{l|}{$O(1 / \sqrt{n})$}         & $\Theta(n)$           &  &  \\
                                               &                                         &                       &  & 
\end{tabular}
\caption{Two illustrative scaling regimes.}
\end{table}

In both regimes the cross entropy loss vanishes, yet the guidance error does not converge to zero.

\subsection{Proof of \cref{theorem: good-classifier-preserve-guidance}}
\label{sec:proof:theorem: good-classifier-preserve-guidance}

For $t \in \{s_0, s_1, \cdots, s_{N - 1}\}$, let $u_t(x) = \log p_{t}(y \mid x) - \log \hat{p}_{t}(y \mid x)$. 
We aim to bound $\mathbb{E}_{x \sim p_t(\cdot \mid y)} \left[ \norm{\nabla_{x}u_t(x)}_{2}^{2} \right]$.
Note that
\begin{align*}
    \nabla_{x} \cdot (u_t(x)p_t(x \mid y)\nabla_{x}u_t(x)) = p_t(x \mid y) \norm{\nabla_{x}u_t(x)}_{2}^{2} + u_t(x) \langle \nabla_{x}p_t(x \mid y), \nabla_{x}u_t(x) \rangle + p_t(x \mid y) u_t(x) \mathrm{Tr}(\nabla_{x}^{2}u_t(x)).
\end{align*}
Consider a closed ball $\mathcal{B}_{\rho}$ with radius $\rho > 0$.
Rearranging and using the fact that $p_t(x \mid y) \nabla_{x} \log p_t(x \mid y) = \nabla_{x}p_t(x \mid y)$, we find that
\begin{align*}
        \int_{\mathcal{B}_{\rho}} \norm{\nabla_{x}u_t(x)}_{2}^{2} p_{t}(x \mid y) \dd x &= \int_{\mathcal{B}_{\rho}} \nabla_{x} \cdot (u_t(x)p_{t}(x \mid y)\nabla_{x}u_t(x)) \dd x \\
        &- \int_{\mathcal{B}_{\rho}} p_{t}(x \mid y) u_t(x) [ \langle \nabla_{x} \log p_{t}(x \mid y), \nabla_{x}u_t(x) \rangle +\mathrm{Tr}(\nabla_{x}^{2}u_t(x))] \dd x. 
    \end{align*}
    Letting $\rho \to \infty$, by Lemmas \ref{prop:ibp-boundary-disappearance} and \ref{propn:ibp-mct-dct}, we find that
    \begin{align*}
        \mathbb{E}_{x \sim p_t(\cdot \mid y)} \left[ \norm{\nabla_{x}u_t(x)}_{2}^{2} \right] = - \underbrace{\int_{\mathbb{R}^{d}} p_{t}(x \mid y) u_t(x) \langle \nabla_{x} \log p_{t}(x \mid y), \nabla_{x}u_t(x) \rangle \dd x}_{\mathrm{I}} - \underbrace{\int_{\mathbb{R}^{d}} p_{t}(x \mid y) u_t(x) \mathrm{Tr}(\nabla_{x}^{2}u_t(x)) \dd x}_{\mathrm{II}}. 
    \end{align*}
We then upper bound terms $\mathrm{I}$ and $\mathrm{II}$, separately.
We first bound term $\mathrm{I}$. To this end, we define the set
\begin{align*}
     \mathcal{G}_t(\epsilon)
        = \Bigl\{
        \|\nabla_x \log p_t(x \mid y)\|_2 \le r_t(\epsilon), \,\, \|\nabla_x \log p_t(y\mid x)\|_2 \le r_t(\epsilon), \,\, \|\nabla_x \log \hat p_{t}(y\mid x)\|_2 \le r_t(\epsilon)
        \Bigr\}, 
\end{align*}
where $r_{t}(\epsilon) = \frac{C \sqrt{d + \log(1/\epsilon)}}{\sigma_{t}}$ is taken to be the maximum of the upper bounds from Corollaries \ref{lemma:gradient-high-prob-bound}, \ref{corollary:gradient-high-prob}, and the theorem condition 3(a). 
We then have $p_t(\mathcal{G}_t(\epsilon)^c \mid y) \leq 3 \epsilon / P_{\rm data}(y)$.
Note that 
\begin{align*}
        \mathbb{E}_{x \sim p_t(\cdot \mid y)} \left[ |u_t(x)| \norm{\nabla_{x} \log p_{t}(x \mid y)}_{2} \norm{\nabla_{x} u_t(x)}_{2} \right] &= \mathbb{E}_{x \sim p_t(\cdot \mid y)} [ |u_t(x)| \norm{\nabla_{x} \log p_{t}(x \mid y)}_{2} \norm{\nabla_{x} u_t(x)}_{2}\mathbbm{1}_{\mathcal{G}_{t}(\epsilon)}] \\
        &+ \mathbb{E}_{x \sim p_t(\cdot \mid y)} [ |u_t(x)| \norm{\nabla_{x} \log p_{t}(x \mid y)}_{2} \norm{\nabla_{x} u_t(x)}_{2}\mathbbm{1}_{\mathcal{G}_{t}(\epsilon)^{c}}]. 
\end{align*}
On $\mathcal{G}_{t}(\epsilon)$, by Lemma \ref{propn:bounding-expectation-of-u} we have
    \begin{align}
    \label{eq:I-A}
    \begin{split}
        \mathbb{E}_{x \sim p_t(\cdot \mid y)} [ |u_t(x)| \norm{\nabla_{x} \log p_{t}(x \mid y)}_{2} \norm{\nabla_{x} u_t(x)}_{2}\mathbbm{1}_{\mathcal{G}_{t}(\epsilon)}] \leq  \,  \frac{2C\varepsilon_t r_t(\epsilon)^2}{P_{\rm data}(y)} \leq  \,\frac{C'\varepsilon_t (d + \log(1/\epsilon))}{P_{\rm data}(y)\sigma_t^2}, 
    \end{split}
    \end{align}
where $C, C' > 0$ are numerical constants. 
On $\mathcal{G}_{t}(\epsilon)^{c}$, by Lemma \ref{propn:tweedie-classifier-conditional} we have 
    \begin{align}
    \label{eq:I-B}
        \mathbb{E}_{x \sim p_t(\cdot \mid y)} [ |u_t(x)| \norm{\nabla_{x} \log p_{t}(x \mid y)}_{2} \norm{\nabla_{x} u_t(x)}_{2}\mathbbm{1}_{\mathcal{G}_{t}(\epsilon)^{c}}] \leq \frac{4\lambda_t R}{\sigma_{t}^{2}} \mathbb{E}_{x \sim p_t(\cdot \mid y)} 
        [ |u_t(x)| \norm{\nabla_{x} \log p_{t}(x \mid y)}_{2} \mathbbm{1}_{\mathcal{G}_{t}(\epsilon)^{c}}].
    \end{align}
Since $P_{\rm data}$ is supported on the ball centered at the origin with radius $R$,
then there exists $b \leq C''( R + \sqrt{d})$ for some positive numerical constant $C''$, such that $p_t(\mathcal{B}_b \mid y) \geq 1 / 2$. By Lemma \ref{propn:bounding-expectation-of-u}, we know that $\E_{x \sim p_t(\cdot \mid y)} [|u_t(x)|] \leq C\varepsilon_t / P_{\rm data}(y)$. 
Therefore, 
there exists $x_{\ast} \in \mathcal{B}_b$, such that $|u_t(x_{\ast})| \leq 2C\varepsilon_t  / P_{\rm data}(y)$. 
Using the fundamental theorem of calculus and Lemma \ref{propn:tweedie-classifier-conditional}, we see that 
    \begin{align}
    \label{eq:distance}
        |u_t(x)| \leq |u_t(x_{\ast})| + \int_{0}^{1} |\langle \nabla u_t(rx + (1 - r) x_{\ast}), x - x_{\ast} \rangle| \dd r \leq |u_t(x_{\ast})| + 4\lambda_{t}\sigma_{t}^{-2}R \norm{x - x_{\ast}}_{2}. 
    \end{align}
By \cref{eq:lemma-b1}, we have
 \begin{align*}
     \|\nabla_x \log p_t(x \mid y) \|_2 = \| \sigma_t^{-2} (\lambda_t m_t(x, y) - x) \|_2 \leq \sigma_t^{-2} \lambda_t R + \sigma_t^{-2} \|x\|_2. 
 \end{align*}
Therefore, 
\begin{align}
\label{eq:I-C}
    & \mathbb{E}_{x \sim p_t(\cdot \mid y)} 
        [ |u_t(x)| \norm{\nabla_{x} \log p_{t}(x \mid y)}_{2} \mathbbm{1}_{\mathcal{G}_{t}(\epsilon)^{c}}] \nonumber\\
        \lesssim & \mathbb{E}_{x \sim p_t(\cdot \mid y)} 
         \big[ (\lambda_t^2 \sigma_t^{-4} R^2 \|x - x_{\ast}\|_2 + \lambda_t \sigma_t^{-4} R \|x\|_2 \|x - x_{\ast}\|_2 +  \sigma_t^{-2} \lambda_t R \varepsilon_t / P_{\rm data}(y)  + \sigma_t^{-2} \|x\|_2 \varepsilon_t / P_{\rm data}(y)) \mathbbm{1}_{\mathcal{G}_{t}(\epsilon)^{c}} \big] \nonumber \\
         \lesssim & \mathbb{E}_{x \sim p_t(\cdot \mid y)} 
         \Big[ \lambda_t^4 \sigma_t^{-8} R^4 \|x - x_{\ast}\|_2^2 + \lambda_t^2 \sigma_t^{-8} R^2 \|x\|_2^2 \|x - x_{\ast}\|_2^2 + \frac{\varepsilon_t^2 \sigma_t^{-4} \lambda_t^2 R^2}{P_{\rm data}(y)^2} + \frac{\varepsilon_t^2 \sigma_t^{-4} \|x\|_2^2}{P_{\rm data}(y)^2}\Big]^{1/2} p_t(\mathcal{G}_t(\epsilon)^c \mid y)^{1/2} \nonumber \\
         \lesssim & \frac{R(R^2 + d) \epsilon^{1/2}}{\sigma_t^4 P_{\rm data}(y)^{1 / 2}} + \frac{(R^2 + d)^{1/2} \varepsilon_t \epsilon^{1/2}}{\sigma_t^{2}P_{\rm data}(y)^{3 / 2}}, 
\end{align}
where ``$\lesssim$'' hides numerical constants. 
Combining \cref{eq:I-A,eq:I-B,eq:I-C}, we get
    \begin{align}
    \label{eq:term-I}
        | \mathrm{I} | \lesssim \frac{\varepsilon_t (d + \log(1/\epsilon))}{P_{\rm data}(y)\sigma_t^2} + \frac{R^2(R^2 + d) \epsilon^{1/2}}{\sigma_t^6 P_{\rm data}(y)^{1 / 2}} + \frac{R(R^2 + d)^{1/2} \varepsilon_t  \epsilon^{1/2}}{\sigma_t^{4}P_{\rm data}(y)^{3 / 2}}. 
    \end{align}
We next bound term $\mathrm{II}$. 
Similarly, we define the set
\begin{align*}
     \mathcal{H}_t(\epsilon)
        = \Bigl\{\,
        \big| \mathrm{Tr}(\nabla_x^{2} \log p_t(x\mid y)) \big| \le R_t(\epsilon),\, \;
        \big| \mathrm{Tr}(\nabla_x^{2} \log p_t(y\mid x)) \big| \le R_t(\epsilon),\,\;
        \big| \mathrm{Tr}(\nabla_x^{2} \log \hat p_{t}(y\mid x)) \big| \le R_t(\epsilon)
        \Bigr\}, 
\end{align*}
where $R_t(\epsilon) = \frac{C' ({d + \log(1/\epsilon)})}{\sigma_{t}^2}$ is taken to be the maximum of the upper bound from Lemma \ref{lemma:hessian-high-prob-bound}, Corollary \ref{corollary:hessian-high-prob}, and the theorem condition 3(b).
In this case, $p_t(\mathcal{H}_{t}(\epsilon)^c \mid y) \leq 3 \epsilon / P_{\rm data}(y)$.
On $\mathcal{H}_{t}(\epsilon)$, by  
Corollary \ref{corollary:hessian-high-prob} and Lemma \ref{propn:bounding-expectation-of-u}, 
    \begin{align*}
        \mathbb{E}_{x \sim p_t(\cdot \mid y)} [ |u_t(x)| \cdot |\mathrm{Tr}(\nabla_{x}^{2}u_t(x))| \mathbbm{1}_{\mathcal{H}_{t}(\epsilon)}] \lesssim \frac{\varepsilon_{t} (d + \log(1/\epsilon))}{\sigma_{t}^{2}P_{\rm data}(y)}.
    \end{align*}
On $\mathcal{H}_{t}(\epsilon)^{c}$, by Lemma \ref{propn:tweedie-classifier-conditional} and \cref{eq:distance} we have 
\begin{align*}
    \mathbb{E}_{x \sim p_t(\cdot \mid y)} [ |u_t(x)| \cdot |\mathrm{Tr}(\nabla_{x}^{2}u_t(x))| \mathbbm{1}_{\mathcal{H}_{t}(\epsilon)^c}] \lesssim & \frac{\lambda_t^2 R^2}{\sigma_t^4} \mathbb{E}_{x \sim p_t(\cdot \mid y)}\big[ (|u_t(x_{\ast})| + 4\lambda_{t}\sigma_{t}^{-2}R \norm{x - x_{\ast}}_{2})^2 \big]^{1/2} \cdot \frac{\epsilon^{1/2}}{P_{\rm data}(y)^{1/2}} \\
    \lesssim & \frac{ R^2 \epsilon^{1/2} \varepsilon_t }{\sigma_t^4 P_{\rm data}(y)^{3 / 2}} + \frac{ R^2(R^2 + d) \epsilon^{1/2}  }{\sigma_t^6 P_{\rm data}(y)^{1/2}}. 
\end{align*}
Next, we put together the above two equations, and conclude that 
\begin{align}
\label{eq:term-II}
    |\mathrm{II}| \lesssim \frac{\varepsilon_{t} (d + \log(1/\epsilon))}{\sigma_{t}^{2}P_{\rm data}(y)} + \frac{ R^2 \epsilon^{1/2} \varepsilon_t }{\sigma_t^4 P_{\rm data}(y)^{3 / 2}} + \frac{ R^2(R^2 + d) \epsilon^{1/2}  }{\sigma_t^6 P_{\rm data}(y)^{1/2}}.
\end{align}
 Combining \cref{eq:term-I,eq:term-II}, we conclude that 
\begin{align*}
    \mathbb{E}_{x \sim p_t(\cdot \mid y)} \left[ \norm{\nabla_{x}u_t(x)}_{2}^{2} \right] \lesssim \frac{\varepsilon_t (d + \log(1/\epsilon))}{P_{\rm data}(y)\sigma_t^2} + \frac{R^2(R^2 + d) \epsilon^{1/2}}{\sigma_t^6 P_{\rm data}(y)^{1 / 2}} + \frac{R(R^2 + d)^{1/2} \varepsilon_t \epsilon^{1/2}}{\sigma_t^{4}P_{\rm data}(y)^{3 / 2}}.
    \end{align*}
Setting 
\begin{equation*}
\epsilon \;=\; \,
\min\!\left\{
\frac{\sigma_t^{8} \varepsilon_{t}^{2}}{R^{4}(R^2 + d)^{2}P_{\rm data}(y)},
\;
\frac{\sigma_t^{4}P_{\rm data}(y)}{R^{2}(R^2 + d)}
\right\},
\end{equation*}
we conclude that 
\begin{align*}
    \mathbb{E}_{x \sim p_t(\cdot \mid y)} \left[ \norm{\nabla_{x}u_t(x)}_{2}^{2} \right] \lesssim \frac{\varepsilon_t}{\sigma_t^2 P_{\rm data}(y)} \Big( d + \log \Big( \frac{R^2 + d}{\sigma_t^2} \Big) + \log \Big( \frac{1}{\varepsilon_t} \Big) + \log \Big( \frac{1}{P_{\rm data}(y)}  \Big) \Big). 
\end{align*}
The proof is done.

\subsection{Proof of Lemma \ref{lem:eps_guide_rate}}
\label{sec:proof:lem:eps_guide_rate}

By \Cref{theorem: good-classifier-preserve-guidance}, we see that
\begin{align*}
\varepsilon_{\mathrm{guide}}^{2}
\lesssim \sum_{k = 0}^{N - 1} \tau_k \cdot 
\frac{\varepsilon_{s_k}}{\sigma_{s_k}^2 P_{\rm data}(y)} \Big( d + \log \Big( \frac{R^2 + d}{\sigma_{s_k}^2} \Big) + \log \Big( \frac{1}{\varepsilon_{s_k}} \Big) + \log \Big( \frac{1}{P_{\rm data}(y)}  \Big) \Big),
\end{align*}
where ``$\lesssim$'' hides a numerical constant. 
When $\sigma_{s_k}^2 \in [1 / 2, 1)$, we have $\tau_k \leq \kappa$, hence
\begin{align*}
    & \tau_k \cdot 
\frac{\varepsilon_{s_k}}{\sigma_{s_k}^2 P_{\rm data}(y)} \Big( d + \log \Big( \frac{R^2 + d}{\sigma_{s_k}^2} \Big) + \log \Big( \frac{1}{\varepsilon_{s_k}} \Big) + \log \Big( \frac{1}{P_{\rm data}(y)}  \Big) \Big) \\
& \lesssim \frac{\kappa \varepsilon_{s_k}}{ P_{\rm data}(y)} \Big( d + \log \Big( \frac{R^2 + d}{\delta} \Big) + \log \Big( \frac{1}{\varepsilon_{s_k}} \Big) + \log \Big( \frac{1}{P_{\rm data}(y)}  \Big) \Big). 
\end{align*}
On the other hand, when $\sigma_{s_k}^2 \in (0, 1/2)$, we have $\tau_k \leq \kappa (T - t_{k + 1}) \leq \kappa s_k$ and $\sigma_{s_k}^{-2} \lesssim s_k^{-1}$. 
Therefore, 
\begin{align*}
    & \tau_k \cdot 
\frac{\varepsilon_{s_k}}{\sigma_{s_k}^2 P_{\rm data}(y)} \Big( d + \log \Big( \frac{R^2 + d}{\sigma_{s_k}^2} \Big) + \log \Big( \frac{1}{\varepsilon_{s_k}} \Big) + \log \Big( \frac{1}{P_{\rm data}(y)}  \Big) \Big) \\
& \lesssim \tau_k \cdot 
\frac{\varepsilon_{s_k}}{s_k P_{\rm data}(y)} \Big( d + \log \Big( \frac{R^2 + d}{\sigma_{s_k}^2} \Big) + \log \Big( \frac{1}{\varepsilon_{s_k}} \Big) + \log \Big( \frac{1}{P_{\rm data}(y)}  \Big) \Big) \\
& \lesssim  
\frac{\kappa \varepsilon_{s_k}}{ P_{\rm data}(y)} \Big( d + \log \Big( \frac{R^2 + d}{\delta} \Big) + \log \Big( \frac{1}{\varepsilon_{s_k}} \Big) + \log \Big( \frac{1}{P_{\rm data}(y)}  \Big) \Big). 
\end{align*}
We complete the proof by combining these two upper bounds.

\subsection{Proof of Theorem \ref{prop:sharpness}}
\label{sec:proof:prop:sharpness}

We first construct a one-dimensional example and then extend it to $d>1$.

\paragraph{Data distribution.}
Let the label space be $\gY = \{\pm 1\}$, and take the uniform label prior $P_{\rm data}(Y=1) = P_{\rm data}(Y=-1) = 1/2$.
Without loss, we assume that the guidance is towards label $1$. 
Let $X_0 \in \R$ be a random variable that is independent of $Y$ with distribution $X_0 \sim \mathrm{Unif}([-R,R])$
for some fixed $R>2$. 
Let $(X_t)_{t\in[0,T]}$ be the OU process starting at $X_0$ (recall this is defined in \cref{eq:ou-forward}). 
For any $t>0$, $X_t$ has a smooth density $p_t$ over $\R$ given by a Gaussian convolution of a compactly-supported density. In particular, for $t > 0$: 
\begin{itemize}
  \item[(1)] $p_t$ is symmetric about $0$ since $X_0$ is symmetric and the OU kernel is Gaussian and also symmetric;
  \item[(2)] $p_t(x) > 0$ for all $x\in\R$ and is continuous, so $p_t$ is bounded away from zero on any compact interval.
\end{itemize}

Because $Y$ and $X_0$ are independent, the OU process is also independent of the label. As a consequence, for every $t \in [0, T]$ and $x \in \mathbb{R}$,
\[
  p_t(Y=1\mid X_0 = x) = \frac12, \;\; p_t(Y=-1\mid X_0 = x) = \frac12.
\]
Thus, we conclude that: 
\begin{itemize}
  \item Assumption~(1) holds: $P_{\mathrm{data}}$ is supported on the compact set $\gK = [-R,R]$ with $R<\infty$.
  \item Assumption~(4) holds with $P_{\mathrm{data}}(y) = 1 / 2$.
\end{itemize}
Moreover, $\log p_t(y\mid x)$ is constant in $x$, so $\frac{\dd}{\dd x} \log p_t(y\mid x) \equiv 0$ and $\frac{\dd^2}{\dd x^2} \log p_t(y\mid x) \equiv 0$. 
\paragraph{Classifier construction.}
Fix $t_0 \in \{s_k: k \in \{0\} \cup [N - 1]\}$. As noted above, $p_{t_0}$ is smooth, symmetric and strictly positive on $\R$. For a sufficiently small $\varepsilon>0$, define 
\[
  \alpha_\varepsilon(x) = \varepsilon \sin\!\Big(\frac{x}{\sqrt{\varepsilon}}\Big),
  \;\; s_\varepsilon(x) = 2\alpha_\varepsilon(x).
\]
Taking the first- and second-order derivatives, we get
\[
  \alpha_\varepsilon'(x)
  = \sqrt{\varepsilon}\,\cos\!\Big(\frac{x}{\sqrt{\varepsilon}}\Big),
  \;\;
  \alpha_\varepsilon''(x)
  = -\sin\!\Big(\frac{x}{\sqrt{\varepsilon}}\Big),
\]
hence
\[
  \sup_{x \in \mathbb{R}} |\alpha_\varepsilon'(x)| \le \sqrt{\varepsilon} \le 1,
  \;\;
  \sup_{x \in \mathbb{R}} |\alpha_\varepsilon''(x)| \le 1
\]
for all sufficiently small $\varepsilon$. 
This implies that $\alpha_\varepsilon$ has uniformly bounded first- and second-order derivatives.

Define a probabilistic classifier at time $t_0$ by
\[
  \hat p_{\varepsilon,t_0}(y=1\mid x) = \sigma\big(s_\varepsilon(x)\big),
  \;\;
  \hat p_{\varepsilon,t_0}(y=-1\mid x) = 1 - \hat p_{\varepsilon,t_0}(y = 1\mid x),
\]
where $\sigma(z) = (1+e^{-z})^{-1}$ is the logistic function. Note that the map $x\mapsto \hat p_{\varepsilon,t_0}(y\mid x)$ belongs to $C^\infty(\mathbb{R})$. 


\paragraph{Verifying classifier smoothness.}
We then verify the classifier smoothness assumption. 
Note that
\[
  \log \hat p_{\varepsilon,t_0}(y=1\mid x) = \log \sigma(s_\varepsilon(x)).
\]
Differentiating and using the equality $\frac{d}{ds}\log\sigma(s) = 1-\sigma(s)$, we get
\[
  \frac{\dd}{\dd x}\log \hat p_{\varepsilon,t_0}(y=1\mid x)
  = (1-\sigma(s_\varepsilon(x)))\,s_\varepsilon'(x).
\]
Since $|s_\varepsilon'(x)| = 2|\alpha_\varepsilon'(x)| \le 2$ and $0 < 1-\sigma(s_\varepsilon(x)) < 1$, we have
\[
  \sup_{x \in \mathbb{R},\varepsilon \in (0, 1/2), t_0 \in \{s_k: k \in \{0\} \cup [N - 1]\}}\Big|\frac{\dd}{\dd x} \log \hat p_{\varepsilon,t_0}(y=1\mid x)\Big|
  \le 2.
\]
Computing the second derivative gives
\[
  \frac{\dd^2}{\dd x^2}\log \hat p_{\varepsilon,t_0}(y=1\mid x)
  = (1-\sigma(s_\varepsilon(x)))\,s_\varepsilon''(x)
    - \sigma(s_\varepsilon(x))(1-\sigma(s_\varepsilon(x)))\,[s_\varepsilon'(x)]^2.
\]
Since $|s_\varepsilon''(x)| = 2|\alpha_\varepsilon''(x)|\le 2$, $|s_\varepsilon'(x)|\le 2$ and $\sigma(s)(1-\sigma(s))\le 1/4$, we get a uniform upper bound
\[
  \sup_{x \in \mathbb{R},\varepsilon \in (0, 1/2), t_0 \in \{s_k: k \in \{0\} \cup [N - 1]\}} \Big|\frac{\dd^2}{\dd x^2} \log \hat p_{\varepsilon,t_0}(y=1\mid x)\Big|
  \le 3.
\] 
Hence, the estimated probabilities have uniformly bounded first- and second-order derivatives, independent of $\varepsilon$, as claimed.
Therefore, 
\begin{itemize}
    \item Assumptions (2) and (3) hold. 
\end{itemize}

\paragraph{Scaling of the conditional KL.}
We now compute the conditional KL divergence achieved by $\hat p_{\varepsilon,t_0}$. 
For each $x \in \mathbb{R}$, the true label conditional distribution is $\mathrm{Bernoulli}(1/2)$, 
while the classifier $\hat p_{\varepsilon,t_0}$ yields a conditional distribution $\mathrm{Bernoulli}(\sigma(s_\varepsilon(x)))$. 
A Taylor expansion around $0$ implies that, for sufficiently small $s$, 
\[
  \KL\big(\mathrm{Bern}(1/2)\,\Vert\,\mathrm{Bern}(\sigma(s))\big)
  \leq c\,s^2
\]
for some numerical constant $c>0$. 
Since $|s_\varepsilon(x)| \le 2\varepsilon$, hence for a sufficiently small $\varepsilon$, 
\[
  \KL\big(p_{t_0}(\cdot\mid x)\,\Vert\,\hat p_{\varepsilon,t_0}(\cdot\mid x)\big)
  \leq c\,s_\varepsilon(x)^2.
\]
Taking the expectation over $X_{t_0}\sim p_{t_0}(\cdot)$ yields the following:
\[
  \E_{X_{t_0}\sim p_{t_0}(\cdot)}\Big[\KL\big(p_{t_0}(\cdot\mid X_{t_0})\,\Vert\,\hat p_{\varepsilon,t_0}(\cdot\mid X_{t_0})\big)\Big]
  \leq c\,\E[s_\varepsilon(X_{t_0})^2] \leq 4c \varepsilon^2.
\]
This bound is independent of $t_0$, hence 
\begin{itemize}
    \item Assumption (5) holds and requirement (i) is satisfied. 
\end{itemize}

\paragraph{Scaling of the guidance term MSE.}
We then analyze the guidance term MSE at time $t_0$.
Note that
\[
  \log p_{t_0}(y=1\mid x) \equiv \log\tfrac12
  \quad\Longrightarrow\quad
  \frac{\dd}{\dd x} \log p_{t_0}(y=1\mid x) = 0.
\]
Since $|s_\varepsilon(x)| \le 2\varepsilon$, there exist absolute constants $0<c_1<c_2<1$, such that for all small enough $\varepsilon$,
\[
  c_1 \;\le\; 1-\sigma(s_\varepsilon(x)) \;\le\; c_2
  \;\; \text{for all } x\in\R.
\]
As a consequence, 
\[
  \Big|\frac{\dd}{\dd x} \log p_{t_0}(y=1\mid x)
        - \frac{\dd}{\dd x} \log \hat p_{\varepsilon,t_0}(y=1\mid x)\Big|
  = \big|(1-\sigma(s_\varepsilon(x)))\,s_\varepsilon'(x)\big|
  \geq c_1 |s_\varepsilon'(x)|.
\]

We next compute the MSE under $X_{t_0}\sim p_{t_0}(\cdot\mid y=1)$. 
Recall that $X_{t_0}$ is independent of $Y$,
therefore, 
\[
  \E_{X_{t_0}\sim p_{t_0}(\cdot\mid y=1)}\big[
    \|\frac{\dd}{\dd x} \log p_{t_0}(y = 1\mid X_{t_0})
      - \frac{\dd}{\dd x} \log \hat p_{\varepsilon,t_0}(y=1\mid X_{t_0})\|_2^2
  \big]
  \geq c_1^2 \E_{X_{t_0} \sim p_{t_0}}\big[s_\varepsilon'(X_{t_0})^2\big].
\]
Using the equality $\cos^{2}(z) = \frac{1}{2}(1 + \cos(2z))$ and the Riemann-Lebesgue lemma, we have
\[
  \E_{X_{t_0}\sim p_{t_0}}\Big[\cos^2\!\Big(\frac{X_{t_0}}{\sqrt{\varepsilon}}\Big)\Big]
  \to \frac12,
  \;\; \varepsilon\to 0 \;\;  \Longrightarrow \;\; \frac{1}{\varepsilon}\E_{X_{t_0}\sim p_{t_0}}\big[s_\varepsilon'(X_{t_0})^2\big]
  \to 2,
  \;\; \varepsilon\to 0. 
\]
This verifies requirement~(ii).

\paragraph{Extension to $d>1$.}
For $d>1$, take $X_0$ to be the uniform distribution over the cube $[-R,R]^d$ that is independent of $Y$. For $x \in \mathbb{R}^d$, define
\[
  \alpha_\varepsilon(x) = \varepsilon \sin\!\Big(\frac{x_1}{\sqrt{\varepsilon}}\Big),
  \;\;
  s_\varepsilon(x) = 2\alpha_\varepsilon(x),
\]
which depend only on the first coordinate of $x$. 
Similarly, one can verify that by defining $\hat p_{\varepsilon, t_0} = \log \sigma(s_{\varepsilon}(x))$, all assumptions (1)--(5) and requirements (i) and (ii) continue to hold.


\subsection{Proof of Theorem \ref{thm:discretization}}
\label{sec:proof:thm:discretization}
Recall the $\gamma$-tilted distribution $p_{T,\gamma}$ from \eqref{eq:gamma-tilted-dist}. Since the map $u\mapsto u^\gamma$ is convex for $\gamma>1$, Jensen's inequality gives $Z_{T, \gamma} \geq \E_{p_{T}}[p_{T}(y \mid X)]^{\gamma} = P_{\mathrm{data}}^{\gamma}$. We define the following path measures:
\begin{align}
    &\mathbb{Q}_{1} = \mathrm{Law} \begin{cases}
        X_{0} \sim p_{T, 1} \\
        \dd \overleftarrow{X}_t^{(y)}
    =
    \Big\{
        \overleftarrow{X}_t^{(y)}
        + 2\nabla_x\log p_{T-t}(\overleftarrow{X}_t^{(y)})
        + 2\nabla_x\log p_{T-t}
        (y\mid \overleftarrow{X}_t^{(y)})
    \Big\}\dd t
    + \sqrt{2}\,\dd B_t',
    \end{cases} \\
    &\mathbb{Q}_{\gamma} = \mathrm{Law} \begin{cases}
        X_{0} \sim p_{T, \gamma} \\
        \dd \overleftarrow{X}_t^{(y)}
    =
    \Big\{
        \overleftarrow{X}_t^{(y)}
        + 2\nabla_x\log p_{T-t}(\overleftarrow{X}_t^{(y)})
        + 2\gamma\nabla_x\log p_{T-t}
        (y\mid \overleftarrow{X}_t^{(y)})
    \Big\}\dd t
    + \sqrt{2}\,\dd B_t'.
    \end{cases}
    \label{eq:laws-q1-qgamma}
\end{align}
We now construct a new path measure whose initial distribution satisfies $X_0\sim p_{T,\gamma}$. To this end, we denote $H_t:=p_{T-t}(y\mid x)$ and $G_t:=\nabla_x\log p_{T-t}(y\mid x)$. Consider the following candidate change-of-measure process:
\begin{equation}
    L_{t} := \frac{p_{T, \gamma}(X_0)}{p_{T,1}(X_0)} \exp \left\{ \sqrt{2}(\gamma - 1) \int_{0}^{t} \langle G_{s}, \dd B_{s}^{\prime} \rangle - (\gamma - 1)^2 \int_{0}^{t} \norm{G_{s}}^{2} \dd s\right\}.
    \label{eq:candidate-process}
\end{equation}
Under $\mathbb Q_1$, the OU Fokker--Planck equation imply that $H_t$ satisfies
$-\partial_t H_t = \left(x+2\nabla_x\log p_{T-t}(x)\right)^\top\nabla_x H_t +\Delta_x H_t$. A simple application of Ito's lemma with Fokker-Planck yields
\begin{equation}
    d \log H_{t} = \norm{G_{t}}^{2}\dd t + \sqrt{2} \langle G_{t}, \dd B_{t}^{\prime} \rangle \iff \sqrt{2} \int_{0}^{t} \langle G_{s}, \dd B_{s}^{\prime} \rangle = \log \frac{H_{t}}{H_0} - \int_{0}^{t} \norm{G_s}^2 \dd s.
    \label{eq:ito-lemma-app}
\end{equation}
Using \eqref{eq:ito-lemma-app} with \eqref{eq:candidate-process} gives us:
\begin{equation}
    L_{t} := \frac{p_{T}(y \mid X_{0})^{\gamma - 1}Z_{T, 1}}{Z_{T, \gamma}} \left( \frac{p_{0}(y \mid X_{T})}{p_{T}(y \mid X_0)} \right)^{\gamma -1} \exp \left\{ -\gamma (\gamma - 1) \int_{0}^{T} \norm{G_s}^2 \dd s \right\}.
    \label{eq:ito-lemma-app-2}
\end{equation}
Assuming local square-integrability of $G_t$, the stochastic exponential in \eqref{eq:ito-lemma-app-2} is a nonnegative local martingale under $\mathbb Q_1$. Since the factor $p_{T,\gamma}(X_0)/p_{T,1}(X_0)$ is bounded and $\mathcal F_0$-measurable, $L_t$ is also a nonnegative local martingale. Moreover, by \eqref{eq:ito-lemma-app-2} and the fact that $0\le H_t\le 1$, $L_{t} \leq 1/P_{\mathrm{data}}^{\gamma}$. Hence $L_t$ is a bounded local martingale and therefore a uniformly integrable martingale. Since $\E_{\mathbb Q_1}[L_0]=1$, it follows that $\E_{\mathbb Q_1}[L_t]=1$ for all $t\le T-\delta$. We have thus established sufficient conditions for Girsanov's theorem to hold (cf. Theorem 8 of \cite{chen2023sampling}).

We now define a path measure $\widetilde{\mathbb{Q}}_{\gamma}$ through $\frac{\dd \widetilde{\mathbb{Q}}_{\gamma}}{\dd \mathbb{Q}_1}= L_{t}$. By Girsanov's theorem, under $\widetilde{\mathbb{Q}}_{\gamma}$, the process $B_{t}^{\gamma} = B_{t}^{\prime} - \sqrt{2}(\gamma -1) \int_{0}^{t} G_{s} \; \dd s$ is a $\widetilde{\mathbb{Q}}_{\gamma}$-Brownian motion. Therefore, we have $\dd B_{t}^{\prime} = \dd B_{t}^{\gamma} + \sqrt{2} (\gamma -1) G_{t} \; \dd t$. Plugging this into the SDE under $\mathbb{Q}_1$ \eqref{eq:laws-q1-qgamma} gives us:
\begin{equation}
    \dd \overleftarrow{X}_t^{(y)}
    =
    \Big\{
        \overleftarrow{X}_t^{(y)}
        + 2\nabla_x\log p_{T-t}(\overleftarrow{X}_t^{(y)})
        + 2\gamma\nabla_x\log p_{T-t}
        (y\mid \overleftarrow{X}_t^{(y)})
    \Big\}\dd t
    + \sqrt{2}\,\dd B_t^{\gamma}.
\end{equation}
Further notice that for any Borel measurable $A \subseteq \mathbb{R}^d$, if $X_{0} \sim p_{T, 1}$ then:
\begin{equation}
    \widetilde{\mathbb{Q}}(X_0 \in A) = \E_{\mathbb{Q}_{1}}\left[ L_{t}\mathbbm{1}_{X_0 \in A} \right] = \E_{\mathbb{Q}_{1}}\left[ \frac{p_{T, \gamma}(X_0)}{p_{T, 1}(X_0)}\mathbbm{1}_{X_0 \in A} \right] = \int_{A} p_{T, \gamma}(x) \dd x.
\end{equation}
Thus $X_{0} \sim p_{T, \gamma}$. Notably, this means the path measure $\widetilde{\mathbb{Q}}_{\gamma}$ is defined as follows:
\begin{equation}
    \widetilde{\mathbb{Q}}_{\gamma} = \mathrm{Law} \begin{cases}
        X_{0} \sim p_{T, \gamma} \\
        \dd \overleftarrow{X}_t^{(y)}
    =
    \Big\{
        \overleftarrow{X}_t^{(y)}
        + 2\nabla_x\log p_{T-t}(\overleftarrow{X}_t^{(y)})
        + 2\gamma\nabla_x\log p_{T-t}
        (y\mid \overleftarrow{X}_t^{(y)})
    \Big\}\dd t
    + \sqrt{2}\,\dd B_t^{\gamma}.
    \end{cases}
\end{equation}
Since $\delta>0$, OU smoothing and the bounded-support assumption imply that the drift of the ideal guided reverse SDE is globally Lipschitz in $x$ with at most linear growth uniformly on $[0,T-\delta]$. Thus by Theorem 5.2.1 of \cite{oksendal2010stochastic}, the SDE admits a unique strong solution and is pathwise unique. 
In particular, uniqueness in law holds so $\widetilde{\mathbb{Q}}_{\gamma} = \mathbb{Q}_{\gamma}$. 

Recall that $q_{\gamma,t}$ denotes the marginal distribution at reverse time $t$ of the ideal classifier-guided process \eqref{eq:ideal-guided-sde} with guidance strength $\gamma$ and initialization $X_0\sim p_{T,\gamma}$. In particular, for $\gamma=1$, we have
$q_{1,t}=p_{T-t}(\cdot\mid y)$. Following the same change-of-measure argument as in \cite{benton2024nearly}, we have
\begin{align*}
D_{\mathrm{KL}}\big(q_{\gamma,T-\delta}\,\|\,\hat q_{\gamma,T-\delta}\big)
\lesssim\;&
D_{\mathrm{KL}}\big(p_{T,\gamma}\,\|\,\pi_d\big)
+
\sum_{k=0}^{N-1}
\tau_k
\E_{q_{\gamma,t_k}}
\left[
\left\|
\nabla\log p_{T-t_k}(X)
-
s_\theta(X,T-t_k)
\right\|_2^2
\right]
\nonumber\\
&+
\gamma^2
\sum_{k=0}^{N-1}
\tau_k
\E_{q_{\gamma,t_k}}
\left[
\left\|
\nabla\log p_{T-t_k}(y\mid X)
-
\nabla\log\hat p_{T-t_k}(y\mid X)
\right\|_2^2
\right]
\nonumber\\
&+
\sum_{k=0}^{N-1}
\int_{t_k}^{t_{k+1}}
\E_{q_{\gamma,t}}
\left[
\left\|
\nabla_x\log
\frac{p_{T-t}(x)}{p_{T-t_k}(x)}
+
\gamma\nabla_x\log
\frac{p_{T-t}(y\mid x)}{p_{T-t_k}(y\mid x)}
\right\|_2^2
\right]\dd t
\nonumber\\
\le\;&
\underbrace{
D_{\mathrm{KL}}\big(p_{T,\gamma}\,\|\,\pi_d\big)
}_{\text{\rm (i)}}
+
\underbrace{
\frac{1}{P_{\mathrm{data}}(y)^\gamma}
\sum_{k=0}^{N-1}
\tau_k
\E_{q_{1,t_k}}
\left[
\left\|
\nabla\log p_{T-t_k}(X)
-
s_\theta(X,T-t_k)
\right\|_2^2
\right]
}_{\text{\rm (ii)}}
\nonumber\\
&+
\underbrace{
\frac{\gamma^2}{P_{\mathrm{data}}(y)^\gamma}
\sum_{k=0}^{N-1}
\tau_k
\E_{q_{1,t_k}}
\left[
\left\|
\nabla\log p_{T-t_k}(y\mid X)
-
\nabla\log\hat p_{T-t_k}(y\mid X)
\right\|_2^2
\right]
}_{\text{\rm (iii)}}
\nonumber\\
&+
\underbrace{
\frac{1}{P_{\mathrm{data}}(y)^\gamma}
\sum_{k=0}^{N-1}
\int_{t_k}^{t_{k+1}}
\E_{q_{1,t}}
\left[
\left\|
\nabla_x\log
\frac{p_{T-t}(x)}{p_{T-t_k}(x)}
+
\gamma\nabla_x\log
\frac{p_{T-t}(y\mid x)}{p_{T-t_k}(y\mid x)}
\right\|_2^2
\right]\dd t
}_{\text{\rm (iv)}}.
\label{eq:gamma-error-decomposition}
\end{align*}
where the second inequality follows from the uniform bound on the Radon-Nikodym derivative \eqref{eq:ito-lemma-app-2}. Term (i) is the initialization error, (ii) is the score estimation error, (iii) is the guidance error, and (iv) is the discretization error. It therefore suffices to control terms (i) and (iv). 

Following \cite[Lemma~2]{benton2024nearly}, Lemma~\ref{lemma:benton} controls term (iv), while Lemma~\ref{lemma:initialization} controls term (i). Under Assumption~\ref{assumption:score}, and recalling that $\varepsilon_{\mathrm{guide}}^2$ is defined in \Cref{lem:eps_guide_rate}, terms (ii) and (iii) are controlled by $\varepsilon_{\mathrm{score}}^2$ and $\varepsilon_{\mathrm{guide}}^2$, respectively. Combining these bounds concludes the proof.

\section{Supporting lemmas}
\label{sec:technical_lemmas}

We summarize in this section the technical lemmas required for our proofs.

\vspace{1em}

\begin{lemma}
\label{propn:tweedie-classifier-conditional}
    Suppose $P_{\mathrm{data}}$ is supported on a bounded set $\mathcal{K} \subseteq \mathbb{R}^{d}$ with $R = \sup_{x \in \mathcal{K}} \norm{x}_{2} < \infty$. 
    Recall that $p_t$ is the marginal distribution of process \eqref{eq:ou-forward}.
    Then the following statements hold:
    \begin{enumerate}[label=(\alph*)]
            \item $\norm{\nabla_{x} \log p_{t}(y \mid x)}_{2} \leq 2\lambda_{t}\sigma_{t}^{-2}R$
        \item $|\mathrm{Tr} ( \nabla_{x}^{2} \log p_{t}(y \mid x) )| \leq \lambda_{t}^{2}\sigma_{t}^{-4}R^2$
    \end{enumerate}
\end{lemma}
\begin{proof}[Proof of Lemma \ref{propn:tweedie-classifier-conditional}]

We prove Lemma \ref{propn:tweedie-classifier-conditional} in Appendix \ref{sec:proof:propn:tweedie-classifier-conditional}. 
\end{proof}

\begin{lemma}
\label{prop:ibp-boundary-disappearance}
    Let $\mathcal{B}_{\rho} = \left\{ x \in \mathbb{R}^{d} : \norm{x}_{2} \leq \rho \right\}$. 
    Recall that $p_t$ is the distribution of $\overrightarrow{X}_t$ in process \eqref{eq:ou-forward}, and $u_t(x) = \log p_t(y \mid x) - \log \hat p_{t}(y \mid x)$. 
    Then under the assumptions of Theorem \ref{theorem: good-classifier-preserve-guidance}, for $t \in \{s_k: k \in \{0\} \cup [N - 1]\}$ we have 
    \begin{align*}
        \lim_{\rho \rightarrow \infty} \int_{\mathcal{B}_{\rho}} \nabla_{x} \cdot \left( u_t(x) p_{t}(x \mid y) \nabla_{x}u_t(x) \right) \dd x = 0.
    \end{align*}

\end{lemma}

\begin{proof}[Proof of Lemma \ref{prop:ibp-boundary-disappearance}]
    We prove Lemma \ref{prop:ibp-boundary-disappearance} in Appendix \ref{sec:proof:prop:ibp-boundary-disappearance}. 
\end{proof}

\begin{lemma}
\label{propn:ibp-mct-dct}
    Under the assumptions of \cref{theorem: good-classifier-preserve-guidance}, the following statements hold for $t \in \{s_k: k \in \{0\} \cup [N - 1]\}$:
    \begin{enumerate}[label=(\alph*)]
        \item $\lim_{\rho \rightarrow \infty} \int_{\mathcal{B}_{\rho}} p_{t}(x \mid y) \norm{\nabla_{x} u_t(x)}_{2}^{2} \dd x =  \int_{\mathbb{R}^{d}} p_{t}(x \mid y) \norm{\nabla_{x} u_t(x)}_{2}^{2} \dd x$. 
        \item We define $g_t(x) = u_t(x)p_{t}(x \mid y) [ \langle \nabla_{x} \log p_{t}(x \mid y), \nabla_{x} u_t(x) \rangle + \mathrm{Tr}( \nabla_{x}^{2} u_t(x) ) ]$. 
        Then, it holds that $$\lim_{\rho \rightarrow \infty} \int_{\mathcal{B}_{\rho}}g_t(x) \dd x = \int_{\mathbb{R}^d}g_t(x) \dd x. $$
    \end{enumerate}
\end{lemma}

\begin{proof}[Proof of Lemma \ref{propn:ibp-mct-dct}]
   We prove Lemma \ref{propn:ibp-mct-dct} in Appendix \ref{sec:proof:propn:ibp-mct-dct}.  
\end{proof}

\begin{lemma}
\label{appendix:l2-norm-high-prob-bound-lemma}
Let $X$ be a random vector in $\mathbb{R}^d$ defined on a probability space $(\Omega,\mathcal{T},P)$, and let $\mathcal{F}\subseteq\mathcal{T}$ be a sub $\sigma$-algebra. 
If $X$ is sub-Gaussian, then
    \[
        \mathbb{P} \left( \norm{\mathbb{E}[X \mid \mathcal{F}]}_{2} \geq  2\norm{X}_{\psi_{2}} \sqrt{\log \left( \frac{2 \cdot 5^{d}}{\epsilon} \right)} \right) \leq \epsilon, 
    \]
    where $\|\cdot\|_{\psi_2}$ denotes the sub-Gaussian norm (see \cite{Vershynin_2018}). 
\end{lemma}

\begin{proof}[Proof of Lemma \ref{appendix:l2-norm-high-prob-bound-lemma}]
    We prove Lemma \ref{appendix:l2-norm-high-prob-bound-lemma} in Appendix \ref{sec:proof:appendix:l2-norm-high-prob-bound-lemma}. 
\end{proof}


We apply Lemma ~\ref{appendix:l2-norm-high-prob-bound-lemma} to derive high-probability bounds on the gradients and Hessians, as stated in the following two corollaries.

\vspace{1em}

\begin{corollary}
\label{lemma:gradient-high-prob-bound}
Recall that $p_t$ is the law of $\overrightarrow{X}_t$ in the forward diffusion process \eqref{eq:ou-forward}, then there exists a universal constant $C$, such that 
    \[
        \mathbb{P}_{x \sim p_t(\cdot \mid y)} \left( \norm{\nabla_{x}\log p_{t}(x \mid y)}_{2} \leq \frac{C \sqrt{d + \log(1/\epsilon)}}{\sigma_{t}} \right) \geq 1 - \epsilon.
    \]
\end{corollary}

\begin{proof}[Proof of Corollary \ref{lemma:gradient-high-prob-bound}]
We prove Corollary \ref{lemma:gradient-high-prob-bound} in Appendix \ref{sec:proof:lemma:gradient-high-prob-bound}.
\end{proof}

\vspace{1em}

\begin{lemma}
\label{lemma:hessian-high-prob-bound}
        There exists a universal constant $C> 0$, such that for $t > 0$, 
    \begin{align*}
         \mathbb{P}_{x \sim p_t(\cdot \mid y)} \left( \,\Big|\mathrm{Tr}\big( \nabla_x^2 \log p_t(x \mid y) \big) \Big| \leq \frac{C (d + \log(1/\epsilon))}{\sigma_{t}^2} \right) \geq 1 - \epsilon.
    \end{align*}

\end{lemma}

\begin{proof}[Proof of Lemma \ref{lemma:hessian-high-prob-bound}]

We prove Lemma \ref{lemma:hessian-high-prob-bound} in Appendix \ref{sec:proof:lemma:hessian-high-prob-bound}.

\end{proof}

Similarly, we have the following results. 

\vspace{1em}

\begin{corollary}
\label{corollary:high-prob-bounds-conditional}
    {\crefalias{enumi}{proposition}
    There exists a universal constant $C> 0$, such that for $t > 0$, 
    \begin{enumerate}[label=(\alph*), ref=\theproposition(\alph*)]
            \item $\mathbb{P}_{x \sim p_t(\cdot \mid y)} \left( \norm{\nabla_{x}\log p_{t}(y \mid x)}_{2} \leq \frac{C \sqrt{d + \log(1/\epsilon)}}{\sigma_{t}} \right) \geq 1 - \epsilon / P_{\rm data}(y)$, \label{corollary:gradient-high-prob}
        \item $\mathbb{P}_{x \sim p_t(\cdot \mid y)} \left(\, \Big| \mathrm{Tr}\big( {\nabla_{x}^{2}\log p_{t}(y \mid x)}\big) \Big| \leq \frac{C(d + \log(1/\epsilon))}{\sigma_{t}^2} \right) \geq 1 - \epsilon  / P_{\rm data}(y)$. \label{corollary:hessian-high-prob}
    \end{enumerate}}
\end{corollary}

\begin{proof}[Proof of Corollary \ref{corollary:high-prob-bounds-conditional}]
We prove Corollary \ref{corollary:high-prob-bounds-conditional} in Appendix \ref{sec:proof:corollary:high-prob-bounds-conditional}. 
\end{proof}

\vspace{1em}

\begin{lemma}\label{lem:posterior-KL-bound}
For $t > 0$, we assume that
\[
  \mathbb{E}_{x \sim p_{t}}
  \Big[
    \KL\big(p_t(\cdot\mid x)\,\Vert\,\hat p_{t}(\cdot\mid x)\big)
  \Big]
  \;\le\; \varepsilon^2.
\]
The probabilistic classifier $\hat p_{t}$ induces a conditional distribution over the label space given the covariates. 
Using the same notation, we define the corresponding joint distribution over $(x, y)$ by $\hat p_{t}(x, y) = p_t(x) \hat p_{t}(y \mid x)$. 
Then for every label $y$ with $P_{\rm data}(y) > 0$, it holds that
\[
  \KL\big(p_t(\cdot\mid y)\,\Vert\,\hat p_{t}(\cdot\mid y)\big)
  \;\le\; \frac{\varepsilon^2}{P_{\rm data}(y)}.
\]
\end{lemma}

\begin{proof}[Proof of Lemma \ref{lem:posterior-KL-bound}]
    We prove Lemma \ref{lem:posterior-KL-bound} in Appendix \ref{sec:proof:lem:posterior-KL-bound}. 
\end{proof}

\vspace{1em}

\begin{lemma}\label{propn:bounding-expectation-of-u}
Recall that $u_t(x) = \log p_t(y \mid x) - \log \hat p_{t}(y \mid x)$. 
Under the conditions of Theorem \ref{theorem: good-classifier-preserve-guidance}, for $t \in \{s_k: k \in \{0\} \cup [N - 1]\}$, there exists a numerical constant $C > 0$, such that 
\begin{align*}
    \mathbb{E}_{x \sim p_t(\cdot \mid y)} [ |u_t(x)|] = \int_{\mathbb{R}^{d}} p_{t}(x \mid y) | u_t(x) | \dd x \leq \frac{C\varepsilon_t}{P_{\rm data}(y)}. 
\end{align*} 
\end{lemma}

\begin{proof}[Proof of Lemma \ref{propn:bounding-expectation-of-u}]
    We prove Lemma \ref{propn:bounding-expectation-of-u} in Appendix \ref{sec:proof:propn:bounding-expectation-of-u}.
\end{proof}

\vspace{1em}

\begin{lemma}
\label{lemma:benton}
    Under the assumptions of \cref{thm:discretization}, we have 
    \begin{align*}
        & \sum_{k = 0}^{N - 1} \int_{t_k}^{t_{k + 1}} 
        \E_{x \sim p_{T-t}(\cdot \mid y)}
        \Big[ 
        \big\| 
        \nabla_x \log p_{T - t}(x) 
        + \gamma \nabla_x \log p_{T - t}(y \mid x) 
        - \nabla_x \log p_{T - t_k}(x) 
        - \gamma \nabla_x \log p_{T - t_k}(y \mid x) 
        \big\|_2^2 
        \Big] \dd t \\
        & \lesssim 
        (2\gamma^2 - 2\gamma + 1)
        \Big(
        \kappa^2 (d + R^2) N 
        + \kappa (d + R^2) T
        \Big), 
    \end{align*}
    where ``$\lesssim$'' hides a positive numerical constant. 
\end{lemma}
\begin{proof}[Proof of Lemma \ref{lemma:benton}]
    We prove Lemma \ref{lemma:benton} in Appendix \ref{proof-benton}.
\end{proof}

\begin{lemma}
\label{lemma:initialization}
    Under the assumptions of \cref{thm:discretization}, for $T\geq 1$ and
    $\gamma>1$, we have
    \begin{equation*}
        D_{\mathrm{KL}}\big(p_{T,\gamma}\,\|\,\pi_d\big)
        \lesssim
        \left(
            d+\frac{\gamma^2R^2}{P_{\mathrm{data}}(y)^{2\gamma}}
        \right)e^{-2T},
    \end{equation*}
    where ``$\lesssim$'' hides a positive numerical constant.
\end{lemma}

\begin{proof}[Proof of Lemma \ref{lemma:initialization}]
    We prove Lemma \ref{lemma:initialization} in Appendix \ref{proof-intialization}.
\end{proof}

\vspace{1em}

\subsection{Proof of Lemma \ref{propn:tweedie-classifier-conditional}}
\label{sec:proof:propn:tweedie-classifier-conditional}

    Recall $p_t$ represents the joint distribution of $(Y, \overrightarrow{X}_t)$ (see Section \ref{subsec:classifier-guidance}). 
    With process \eqref{eq:ou-forward}, we define $$m_{t}(x) = \mathbb{E}[\overrightarrow{X}_{0} \mid \overrightarrow{X}_{t} = x], \quad m_{t}(x, y) = \mathbb{E}[\overrightarrow{X}_{0} \mid \overrightarrow{X}_{t} = x, Y=y].$$ 
    Note that $\overrightarrow{X}_{t} \mid \overrightarrow{X}_{0} = x_{0} \sim \mathcal{N}( \lambda_{t}x_{0}, \sigma_{t}^{2}I_{d})$, hence we have 
    \begin{align*}
        \nabla_{x}p_{t}(x \mid x_{0}) = - \left(\frac{x - \lambda_{t}x_{0}}{\sigma_{t}^{2}} \right) p_{t}(x \mid x_{0}).
    \end{align*}
Note that $p_t(x, y) = \int_{\R^d} p_t(x \mid x_0) p(x_0, y) \dd x_0$, hence (one can check that differentiation under the integral sign is valid)
    \begin{align}
    \label{eq:lemma-b1}
    \begin{split}
        \nabla_{x}p_{t}(x, y) &= \int_{\mathbb{R}^{d}} \nabla_{x} p_{t}(x \mid x_{0}) p_{0}(x_{0}, y) \dd x_{0} \\
        &= \frac{1}{\sigma_{t}^{2}} \int_{\mathbb{R}^{d}} (\lambda_{t}x_{0} - x) p_{t}(x \mid x_{0}) p_{0}(x_{0}, y) \dd x_{0} \\
        &= \sigma_{t}^{-2}p_{t}(x, y) \left( \lambda_{t}m_{t}(x, y) - x \right).
    \end{split}
    \end{align}
    Thus $\nabla_{x}\log p_{t}(x, y) = \sigma_{t}^{-2} \left( \lambda_{t}m_{t}(x, y) - x \right)$. Similar rationale gives us $\nabla_{x} \log p_{t}(x) = \sigma_{t}^{-2} \left( \lambda_{t}m_{t}(x) - x \right)$. Therefore, 
    \begin{align*}
        \nabla_{x} \log p_{t}(y \mid x) &= \nabla_{x} \log p_{t}(x, y) - \nabla_{x}\log p_{t}(x) \\
        &= \sigma_{t}^{-2} \left( \lambda_{t}m_{t}(x, y) - x \right) - \sigma_{t}^{-2} \left( \lambda_{t}m_{t}(x) - x \right) \\
        &= \lambda_{t}\sigma_{t}^{-2} \left( m_{t}(x, y) - m_{t}(x)\right). 
    \end{align*}
    By assumption, $P_{\mathrm{data}}$ is supported on a bounded set $\mathcal{K} \subseteq \mathbb{R}^{d}$ with $R = \sup_{x \in \mathcal{K}} \norm{x}_{2} < \infty$, hence $\norm{m_{t}(x_{t}, y)}_{2}, \norm{m_{t}(x_{t})}_{2} \leq R$.
    Using these upper bounds, we establish via triangle inequality that $\norm{\nabla_{x} \log p_{t}(y \mid x)}_{2} \leq 2\lambda_{t}\sigma_{t}^{-2}R$. 
    This proves the first point of the lemma. 

    Next, we compute the Hessian of $p_t(x, y)$ with respect to $x$. 
    Denote $\Sigma_t(x) = \Cov[\overrightarrow{X}_0 \mid \overrightarrow{X}_t = x]$ and $\Sigma_{t}(x, y) = \Cov[\overrightarrow{X}_0 \mid \overrightarrow{X}_t = x, Y = y]$.  
Under the current regularity assumptions, we have
\begin{align*}
    \nabla_x^2 p_t(x, y)
    &= \int_{\mathbb{R}^d} \nabla_x^2 p_t(x \mid x_0)\, p_0(x_0, y)\, \dd x_0 \\
    &= \sigma_t^{-4} \int_{\mathbb{R}^d}
       \Big[
         (x - \lambda_t x_0)(x - \lambda_t x_0)^\top
         - \sigma_t^2 I_d
       \Big]
       p_t(x \mid x_0)\, p_0(x_0, y)\, \dd x_0 \\
       & = \sigma_t^{-4} p_t(x, y)
       \Big(
         \mathbb{E}\big[(x - \lambda_t \overrightarrow{X}_0)(x - \lambda_t \overrightarrow{X}_0)^\top \mid \overrightarrow{X}_t = x, Y = y\big]
         - \sigma_t^2 I_d
       \Big) \\
       & = \sigma_t^{-4} p_t(x, y)
       \Big(
         (x - \lambda_t m_t(x, y))(x - \lambda_t m_t(x, y))^\top
  + \lambda_t^2\, \operatorname{Cov}[\overrightarrow{X}_0 \mid \overrightarrow{X}_t = x, Y = y]
         - \sigma_t^2 I_d
       \Big). 
\end{align*}
We then compute the Hessian of the log-density. Note that
\[
\nabla_x^2 \log p_t(x, y)
= \frac{1}{p_t(x, y)} \nabla_x^2 p_t(x, y)
  - \big(\nabla_x \log p_t(x, y)\big)\big(\nabla_x \log p_t(x, y)\big)^\top,
\]
Usng the expressions we derived for $\nabla_{x}\log p_{t}(x, y)$ and $\nabla_x^2 p_t(x, y)$, we have
\begin{align*}
    \nabla_x^2 \log p_t(x, y)
    &= \lambda_t^2 \sigma_t^{-4}\, \operatorname{Cov}[\overrightarrow{X}_0 \mid \overrightarrow{X}_t = x, Y = y]
       - \sigma_t^{-2} I_d.
\end{align*}
An analogous calculation for $p_t(x)$ gives
\[
\nabla_x^2 \log p_t(x)
= \lambda_t^2 \sigma_t^{-4}\, \operatorname{Cov}[\overrightarrow{X}_0 \mid \overrightarrow{X}_t = x]
  - \sigma_t^{-2} I_d,
\]
so that
\begin{align*}
    \nabla_x^2 \log p_t(y \mid x)
    = \nabla_x^2 \log p_t(x, y) - \nabla_x^2 \log p_t(x) = \lambda_t^2 \sigma_t^{-4}
       \Big(
         \Sigma_{t}(x, y)
         - \Sigma_{t}(x)
       \Big).
\end{align*}
By assumption $\mathrm{Tr}(\Sigma_{t}(x_{t}, y)), \mathrm{Tr}(\Sigma_{t}(x_{t})) \in [0, R^{2}]$, thus we have
\begin{align*}
    \left| \mathrm{Tr} \left( \nabla_{x}^{2} \log p_{t}(y \mid x) \right) \right| = \left|\lambda_t^2 \sigma_t^{-4}
       \Big(
         \mathrm{Tr}(\Sigma_{t}(x, y))
         - \mathrm{Tr}(\Sigma_{t}(x))
       \Big)\right| \leq \lambda_{t}^{2}\sigma_{t}^{-4}R^2.
\end{align*}
The proof is done.

\subsection{Proof of Lemma \ref{prop:ibp-boundary-disappearance}}
\label{sec:proof:prop:ibp-boundary-disappearance}

Note that $p_{t}(x \mid y) = \int_{\mathcal{K}} P_{\rm data}(x_{0} \mid y) \varphi_{t}(x \mid x_{0}) \dd x_0$, 
where $\varphi_t$ is the normal density associated with the conditional distribution $\overrightarrow{X}_{t} \mid \overrightarrow{X}_{0} = x_{0} \sim \mathcal{N}(\lambda_{t}x_{0}, \sigma_{t}^{2}I_{d})$. Therefore, $p_t(x \mid y) \leq \sup_{z \in \mathcal{K}} \varphi_t(x \mid z)$. 
For $\rho \geq 4 \lambda_t R$,  $\|x\|_2 = \rho$ and $z \in \mathcal{K}$, by the triangle inequality we have $\norm{x - \lambda_{t}z}_{2} \geq \rho - \lambda_{t} R$. Thus we have
    \begin{align*}
        \varphi_{t}(x \mid z) &= \frac{1}{(2\pi \sigma_{t}^{2})^{d/2}} \exp \left(  -\frac{\norm{x_{t} - \lambda_{t}z}_{2}^{2}}{2\sigma_{t}^2} \right) \leq \frac{1}{(2\pi \sigma_{t}^{2})^{d/2}} \exp \left(  -\frac{({\rho -  \lambda_{t}R})^{2}}{2\sigma_{t}^2} \right).
    \end{align*}
    Let $C_{t} = (2\pi \sigma_{t}^{2})^{-d/2}$ and $c_{t} = 1 / (2\sigma_{t}^{2})$. We can then conclude that for $\|x\|_2 = \rho$, it holds that $ p_{t}(x \mid y) \leq C_{t} \exp ( -c_{t} \left( \rho - \lambda_{t}  R \right)^{2} )$. For $\rho \geq 4 \lambda_{t}R$, we have $(\rho -  \lambda_{t} R)^{2} \geq \rho^{2} / 2$, which implies that $\sup_{\norm{x}_{2} = \rho} p_{t}(x \mid y) \leq C_{t} \exp \left( -c_{t}\rho^{2} / 2 \right)$. 

    By Proposition \ref{propn:tweedie-classifier-conditional} (a), we have $\norm{\nabla_{x} \log p_{t}(y \mid x)}_{2} \leq 2\lambda_{t}\sigma_{t}^{-2}R$.
    By triangle inequality, this implies that $\norm{\nabla_{x} u_t(x)}_{2} \leq 4\lambda_{t}\sigma_{t}^{-2}R$. 
    By the fundamental theorem of calculus we have
    \begin{align*}
        |u_t(x)| &\leq |u_t(\vec{0}_d)| + \int_{0}^{1} |\langle \nabla_{x}u_t(rx), x \rangle| \dd r  \leq |u_t(\vec{0}_d)| + 4\lambda_{t}\sigma_{t}^{-2}R \norm{x}_{2}.
    \end{align*}
    Since $p_{t}(y \mid x), \hat{p}_{\phi, t}(y \mid x) \in C^{2}(\mathbb{R}^{d})$, the function $x \mapsto u_t(x) p_{t}(x \mid y) \nabla_{x}u_t(x) \in C^{2}(\mathbb{R}^{d})$. By the divergence theorem, it holds that
    \begin{align*}
        I_{\rho} = \int_{\mathcal{B}_{\rho}} \nabla_{x} \cdot \left( u_t(x) p_{t}(x \mid y) \nabla_{x}u_t(x) \right) \dd x = \int_{\partial \mathcal{B}_{\rho}} u_t(x) p_{t}(x \mid y) \nabla_{x} u_t(x) \cdot n(x) \dd S
    \end{align*}
    where $n(x)$ is the outward pointing unit normal at almost each point on the boundary $\partial \mathcal{B}_{\rho}$. 
On $\partial \mathcal{B}_{\rho}$, we have 
\begin{align*}
    |u_t(x)|\, p_{t}(x \mid y)\, \|\nabla_{x} u_t(x)\|_2 \leq 4\lambda_{t}\sigma_{t}^{-2}R \Big(|u_t(\vec{0}_d)| + 4\lambda_{t}\sigma_{t}^{-2}R \norm{x}_{2} \Big) C_{t} \exp \left( -c_{t}\rho^{2} / 2 \right). 
\end{align*}
A classical result in measure theory states that the Lebesgue measure of a $d$-dimensional sphere of radius $\rho$ satisfies $\mu(\partial \mathcal{B}_{\rho}) \asymp \rho^{d-1}$ \cite[Section 2.7, Corollary 2.55]{folland2013real}.
Combining these results, we conclude that
    \begin{align*}
        | I_{\rho} | &\leq \int_{\partial \mathcal{B}_{\rho}} |u_t(x)|\, p_{t}(x \mid y) \norm{\nabla_{x} u_t(x)}_{2} \dd S \\
        &\lesssim 4\lambda_{t}\sigma_{t}^{-2}R \Big(|u_t(\vec{0}_d)| + 4\lambda_{t}\sigma_{t}^{-2}R \norm{x}_{2} \Big) C_{t} \exp \left( -c_{t}\rho^{2} / 2 \right) \rho^{d - 1} \to 0 \;\; \mbox{as }\rho \to \infty. 
    \end{align*}
    The proof is done.

\subsection{Proof of Lemma \ref{propn:ibp-mct-dct}}
\label{sec:proof:propn:ibp-mct-dct}

Define $f_{\rho}(x) = p_{t}(x \mid y) \norm{\nabla_{x} u_t(x)}_{2}^{2} \mathbbm{1}_{\mathcal{B}_{\rho}}(x)$, then as $\rho \to \infty$, 
\begin{align*}
    \lim_{\rho \to \infty} \mathbbm{1}_{\mathcal{B}_{\rho}} = 1, \;\;  \lim_{\rho \to \infty} p_{t}(x \mid y) \norm{\nabla_{x} u_t(x)}_{2}^{2} \mathbbm{1}_{\mathcal{B}_{\rho}}(x) = p_{t}(x \mid y) \norm{\nabla_{x} u_t(x)}_{2}^{2}. 
\end{align*}
By the monotone convergence theorem, 
    \begin{align*}
        \lim_{\rho \rightarrow \infty} \int_{\mathcal{B}_{\rho}} p_{t}(x \mid y) \norm{\nabla_{x} u_t(x)}_{2}^{2} \dd x =  \int_{\mathbb{R}^{d}} p_{t}(x \mid y) \norm{\nabla_{x} u_t(x)}_{2}^{2} \dd x. 
    \end{align*}
    Now consider $g_t(x) = u_t(x)p_{t}(x \mid y) [ \langle \nabla_{x} \log p_{t}(x \mid y), \nabla_{x} u_t(x) \rangle + \mathrm{Tr}( \nabla_{x}^{2} u_t(x) ) ]$. Note that
    \begin{align*}
        | g_t(x) | \lesssim C_{t} \lambda_t^2 \sigma_t^{-4} R^2 \exp \left( -c_{t}\|x\|_2^{2} / 2 \right)  \Big(|u_t(\vec{0}_d)| + 4\lambda_{t}\sigma_{t}^{-2}R \norm{x}_{2} \Big) := h_t(x).
    \end{align*}
    Notice that $h_t \in L^{1}(\mathbb{R}^{d})$, hence by the dominated convergence theorem, 
    \begin{align*}
        \lim_{\rho \rightarrow \infty} \int_{\mathcal{B}_{\rho}}g_t(x) \dd x = \int_{\mathbb{R}^d} g_t(x) \dd x.
    \end{align*}
    The proof is done.

\subsection{Proof of Lemma \ref{appendix:l2-norm-high-prob-bound-lemma}}
\label{sec:proof:appendix:l2-norm-high-prob-bound-lemma}

    Fix $u \in \mathbb{S}^{d-1}$. By Jensen's inequality and Markov's inequality, we have
    \begin{align*}
        \mathbb{P} \left( \mathbb{E} [ \langle X, u \rangle \mid \mathcal{F}] \geq \lambda \right) &\leq \inf_{c > 0}\,  \mathbb{P} \left( e^{\mathbb{E} [ \langle X, u \rangle \mid \mathcal{F}]^2 / c^2} \geq e^{\lambda^2 / c^{2}} \right) \\
        &\leq \inf_{c > 0}\,  \mathbb{P} \left( \mathbb{E} \left[ e^{\langle X, u \rangle^2 / c^2} \mid \mathcal{F} \right] \geq e^{\lambda^2 / c^{2}}  \right) \\
        &\leq \inf_{c > 0}\, \frac{\mathbb{E} \left[ \mathbb{E} \left[ e^{\langle X, u \rangle^2 / c^2} \mid \mathcal{F} \right] \right]}{e^{\lambda^2 / c^{2}}} \\
        &= \inf_{c > 0}\,  \frac{\mathbb{E}[e^{\langle X, u \rangle^2 / c^2} ]}{e^{\lambda^2 / c^{2}}} \\
        &\leq 2 e^{-\lambda^2/\norm{X}_{\psi_{2}}^2}
    \end{align*}
    where the last inequality holds by taking $c = \norm{X}_{\psi_{2}}$ (\cite[Definition 2.6.4]{Vershynin_2018}). 
Let $\mathcal{N}\subseteq\mathbb{S}^{d-1}$ be a $1/2$-net of $\mathbb{S}^{d-1}$. Standard volumetric arguments then imply that $|\mathcal{N}|\le 5^d$ (\cite[Corollary 4.2.11]{Vershynin_2018}).
Then a union bound implies that
    \[
        \mathbb{P} \left( \exists u \in \mathcal{N} : \mathbb{E} [ \langle X, u \rangle \mid \mathcal{F}] \geq \lambda  \right) \leq 5^{d} \cdot 2e^{-\lambda^2/\norm{X}_{\psi_{2}}^2}.
    \]
By \cite[Exercise 4.35]{Vershynin_2018}, we have $\norm{x}_{2} \leq 2 \sup_{u \in \mathcal{N}} \langle u, x \rangle$.
Using this and setting $\lambda = \norm{X}_{\psi_{2}} \sqrt{\log \left( \frac{2 \cdot 5^{d}}{\epsilon} \right)}$, we get 
\begin{align*}
    \mathbb{P} \left( \norm{\mathbb{E}[X \mid \mathcal{F}]}_{2} \geq  2\norm{X}_{\psi_{2}} \sqrt{\log \left( \frac{2 \cdot 5^{d}}{\epsilon} \right)} \right) \leq \epsilon.
\end{align*}
The proof is done.

\subsection{Proof of Corollary \ref{lemma:gradient-high-prob-bound}}
\label{sec:proof:lemma:gradient-high-prob-bound}
    Inspecting the proof of  Lemma \ref{propn:tweedie-classifier-conditional}, we have
    \begin{align*}
        \nabla_{x} \log p_{t}(x \mid y) &= \sigma_{t}^{-2} \left( \lambda_t \mathbb{E}[X_{0} \mid X_{t} = x, Y=y] - x \right) \\
        &= -\sigma_{t}^{-2} \mathbb{E}[X_{t} - \lambda_{t}X_{0} \mid X_{t} = x, Y = y],
    \end{align*}
    where $X_{0} \sim P_{\mathrm{data}}$ and $X_{t} = \lambda_{t}X_{0} + \sigma_{t}\xi$ for $\xi \sim \mathcal{N}(0, I_{d})$ that is independent of $X_0$. Noting that $X_{t} - \lambda_{t}X_{0} = \sigma_t\xi$ is sub-Gaussian with $\norm{X_{t} - \lambda_{t}X_{0}}_{\psi_{2}} \leq C_{1} \sigma_{t}$ for some universal constant $C_{1} > 0$. The proof is done by Lemma \ref{appendix:l2-norm-high-prob-bound-lemma}.

\subsection{Proof of Lemma \ref{lemma:hessian-high-prob-bound}}
\label{sec:proof:lemma:hessian-high-prob-bound}

        Inspecting the proof of  Lemma \ref{propn:tweedie-classifier-conditional}, we have 
    \begin{align*}
        \nabla_x^2 \log p_t(x \mid y) = &\, \lambda_t^2 \sigma_t^{-4}\, \operatorname{Cov}[X_0 \mid {X}_t = x, Y = y]
       - \sigma_t^{-2} I_d \\
       = & \, \sigma_t^{-4}\, \operatorname{Cov}[X_t - \lambda_t X_0 \mid {X}_t = x, Y = y]
       - \sigma_t^{-2} I_d, 
    \end{align*}
    where $X_{0} \sim P_{\mathrm{data}}$ and $X_{t} = \lambda_{t}X_{0} + \sigma_{t}\xi$ for $\xi \sim \mathcal{N}(0, I_{d})$ that is independent of $X_0$. Note that $X_{t} - \lambda_{t}X_{0} = \sigma_t\xi$. 
    Therefore, 
    \begin{align*}
        \Big|\mathrm{Tr}\big( \nabla_x^2 \log p_t(x \mid y) \big) \Big| \leq\, & \sigma_t^{-2} \Big( d +  \mathrm{Tr} \big(\operatorname{Cov}[\xi \mid {X}_t = x, Y = y] \big)  \Big) \\
        \leq\, & \sigma_t^{-2} \Big( d +  \mathrm{Tr} \big(\,\E[\xi \xi^{\top} \mid {X}_t = x, Y = y] \big)  \Big). 
    \end{align*}
    Note that $\E_{x \sim p_t(\cdot \mid y)}[\mathrm{Tr}(\E[\xi \xi^{\top} \mid {X}_t = x, Y = y])] = d$, and that for $\lambda \geq 0$, 
    \begin{align*}
        \E_{x \sim p_t(\cdot \mid y)}\Big[ e^{\lambda (\mathrm{Tr} (\,\E[\xi \xi^{\top} \mid {X}_t = x, Y = y] ) - d)} \Big] \leq  \E_{x \sim p_t(\cdot \mid y)}\Big[ \E\Big[ e^{\lambda (\mathrm{Tr} (\xi \xi^{\top})   - d)} \mid X_t = x, Y = y \Big]\Big], 
    \end{align*}
    where the upper bound is by Jensen's inequality. 
    Therefore, $\|\mathrm{Tr} (\,\E[\xi \xi^{\top} \mid {X}_t = x, Y = y] ) - d\|_{\psi_1} \leq c_1 \|\mathrm{Tr} (\xi \xi^{\top} ) - d\|_{\psi_1} \leq c_2 \sqrt{d}$ (note that $(\mathrm{Tr} (\xi \xi^{\top} ) - d)$ is distributed as a centered chi-squared distribution with $d$ degrees of freedom). 
    Here, $c_1, c_2 > 0$ are numerical constants, and $\|\cdot\|_{\psi_1}$ denotes the sub-exponential norm \cite{Vershynin_2018}. 
    Therefore, by sub-exponential tail bound, there exists a numerical constant $C > 0$, such that  
    \begin{align*}
        \mathbb{P}_{x \sim p_t(\cdot \mid y)} \left( \,\Big|\mathrm{Tr}\big( \nabla_x^2 \log p_t(x \mid y) \big) \Big| \leq \frac{C (d + \log(1/\epsilon))}{\sigma_{t}^2} \right) \geq 1 - \epsilon.
    \end{align*}
    The proof is done.

\subsection{Proof of Corollary \ref{corollary:high-prob-bounds-conditional}}
\label{sec:proof:corollary:high-prob-bounds-conditional}
    We first prove Corollary \ref{corollary:gradient-high-prob}. 
    By Bayes rule and the proof of Lemma \ref{propn:tweedie-classifier-conditional},
    \begin{align*}
        \nabla_{x}\log p_{t}(y \mid x) &= \nabla_{x} \log p_{t}(x, y) - \nabla_{x} \log p_{t}(x) \\
        &= -\sigma_{t}^{-2} \left( \mathbb{E}[X_{t} - \lambda_{t}X_{0} \mid X_{t} = x, Y = y] - \mathbb{E}[X_{t} - \lambda_{t}X_{0} \mid X_{t} = x]  \right),
    \end{align*}
    where $X_{0} \sim P_{\mathrm{data}}$ and $X_{t} = \lambda_{t}X_{0} + \sigma_{t}\xi$ for $\xi \sim \mathcal{N}(0, I_{d})$ that is independent of $X_0$.
    By Corollary \ref{lemma:gradient-high-prob-bound}, there exists a numerical constant $C_1 > 0$, such that 
    \begin{align*}
        \mathbb{P}_{x \sim p_t(\cdot \mid y)} \left( \norm{\nabla_{x}\log p_{t}(x \mid y)}_{2} \leq \frac{C_1 \sqrt{d + \log(1/\epsilon)}}{\sigma_{t}} \right) \geq 1 - \epsilon.
    \end{align*}
    Likewise, there exists a numerical constant $C_2 > 0$, such that 
    \begin{align*}
        & \mathbb{P}_{x \sim p_t} \left( \norm{\nabla_{x}\log p_{t}(x)}_{2} \leq \frac{C_2 \sqrt{d + \log(1/\epsilon)}}{\sigma_{t}} \right) \geq 1 - \frac{\epsilon}{2} \\
        & \Longrightarrow \mathbb{P}_{x \sim p_t(\cdot \mid y)} \left( \norm{\nabla_{x}\log p_{t}(x)}_{2} \leq \frac{C_2 \sqrt{d + \log(1/\epsilon)}}{\sigma_{t}} \right) \geq 1 - \frac{\epsilon}{2P_{\rm data}(y)}. 
    \end{align*}
    Point (a) then follows by taking $C = C_1 + C_2$. 
    The proof of point (b) relies on Lemma \ref{lemma:hessian-high-prob-bound} and is similar to that of point (a), and we skip it for compactness.

\subsection{Proof of Lemma \ref{lem:posterior-KL-bound}}
\label{sec:proof:lem:posterior-KL-bound}

Note that 
\begin{align*}
  \KL\big(p_t\,\Vert\,\hat p_{t}\big)
  = \mathbb{E}_{(x, y') \sim p_t}\!\left[
        \log \frac{p_t(y', x)}{\hat p_{t}(y', x)}
     \right] = \mathbb{E}_{(x, y') \sim p_t}\!\left[
        \log \frac{p_t(y'\mid x)}{\hat p_{t}(y'\mid x)}
     \right] = \mathbb{E}_{x \sim p_{t}} \left[
      \KL\big(p_t(\cdot\mid x)\,\Vert\,\hat p_{t}(\cdot\mid x)\big) \right]. 
\end{align*}
By assumption we have  $\KL\big(p_t\,\Vert\,\hat p_{t}\big)
  \;\le\; \varepsilon^2$. 
Using the chain rule for KL divergence, we get
\begin{align*}
  \KL\big(p_t\,\Vert\,\hat p_{t}\big)
  =  \mathbb{E}_{y \sim p_t}
     \big[\KL\big(p_t( \cdot \mid y)\,\Vert\,\hat p_{t}(\cdot \mid y)\big)\big] + \KL\big( p_t(y) \parallel \hat p_{t}(y) \big),
\end{align*}
which further implies that 
\[
  \sum_{y \in \mathcal{Y}} p_t(y)\,
      \KL\big(p_t(\cdot\mid y)\,\Vert\,\hat p_{t}(\cdot\mid y)\big)
  \;\le\; \varepsilon^2 \Longrightarrow \KL\big(p_t(\cdot\mid y)\,\Vert\,\hat p_{t}(\cdot\mid y)\big)
  \;\le\; \frac{\varepsilon^2}{p_t(y)} = \frac{\varepsilon^2}{P_{\rm data}(y)}.
\]
The proof is done.

\subsection{Proof of Lemma \ref{propn:bounding-expectation-of-u}}
\label{sec:proof:propn:bounding-expectation-of-u}

Define  $A = \{ x : p_{t}(y \mid x) < \hat{p}_{t} (y \mid x) \}$ and $B = \{ x : p_{t}(y \mid x) \geq \hat{p}_{t} (y \mid x) \}$. Note that 
\[
    \mathbb{E}_{x \sim p_t(\cdot \mid y)} [ |u_t(x)|] = \underbrace{\int_{A} p_{t}(x \mid y)| u_t(x)| \dd x}_{\mathrm{I}} + \underbrace{\int_{B} p_{t}(x \mid y) |u_t(x)| \dd x}_{\mathrm{II}}.
\]
We first bound term $\mathrm{I}$. 
On set $A$ we have 
$\log p_t(y \mid x) - \log \hat{p}_{t}(y \mid x) < 0$. Since $\log(1+x) \leq x$ , we have
\[
\left| \log \left( \frac{p_{t}(y \mid x)}{\hat{p}_{t}(y \mid x)}\right) \right|
= \log \left( \frac{\hat{p}_{t}(y \mid x)}{p_{t}(y \mid x)} \right)
\leq \frac{\hat{p}_{t}(y \mid x) - p_{t}(y \mid x)}{p_{t}(y \mid x)}.
\]
Using this upper bound, we can bound term  $\mathrm{I}$ as follows:
\begin{align*}
    \mathrm{I}
    \leq \int_{A} p_{t}(x \mid y)\,
       \frac{\hat{p}_{t}(y \mid x) - p_{t}(y \mid x)}{p_{t}(y \mid x)}
       \, \dd x.
\end{align*}
Applying Bayes' rule $p_t(x \mid y) = p_t(x)\, p_t(y \mid x) / p_t(y)$, the integrand simplifies to:
\[
p_t(x \mid y)\,
\frac{\hat{p}_{t}(y \mid x) - p_{t}(y \mid x)}{p_{t}(y \mid x)}
= \frac{p_t(x)}{p_t(y)}\,
  \big(\hat{p}_{t}(y \mid x) - p_{t}(y \mid x)\big).
\]
Therefore, using the equality $p_t(y) = P_{\rm data}(y)$, it holds that 
\begin{align}
\label{eq:BB2}
    \mathrm{I}
    \leq \frac{1}{p_t(y)} \int_{A} p_t(x)\,
        \big(\hat{p}_{t}(y \mid x) - p_{t}(y \mid x)\big)\,\dd x \leq \frac{1}{P_{\rm data}(y)} \int p_t(x)\,
        \big|\hat{p}_{t}(y \mid x) - p_{t}(y \mid x)\big|\,\dd x.
\end{align}
For each $x$, let $P_{x} = p_t(\cdot \mid x)$ and
$Q_{x} = \hat{p}_{t}(\cdot \mid x)$. Then
\[
\big|\hat{p}_{t}(y \mid x) - p_{t}(y \mid x)\big|
\leq \mathrm{TV}\big(P_{x}, Q_{x}\big)
\leq \sqrt{\frac{1}{2}\,
           \KL\big(P_{x} \,\Vert\, Q_{x}\big)},
\]
where the last upper bound follows from Pinsker's inequality. 
Taking expectations and using Jensen's inequality, we have
\begin{align}
\label{eq:B1}
\begin{split}
\mathbb{E}_{X_t \sim p_t}\big[\,\big|\hat{p}_{t}(y \mid X_{t}) - p_{t}(y \mid X_{t})\big|\,\big]
&\leq \mathbb{E}_{X_t \sim p_t}\left[
    \sqrt{\frac{1}{2}\,
           \KL\big(p_t(\cdot \mid X_t) \,\Vert\, \hat{p}_{t}(\cdot \mid X_t)\big)}
    \right] \\
&\leq \sqrt{\frac{1}{2}\,
             \mathbb{E}_{X_t \sim p_t} \big[
             \KL\big(p_t(\cdot \mid X_t) \,\Vert\, \hat{p}_{t}(\cdot \mid X_t)\big) \big]}.
\end{split}
\end{align}
By assumption, we have $\mathbb{E}_{X_t \sim p_t}[
      \KL(p_t(\cdot\mid X_t)\,\Vert\,\hat p_{t}(\cdot\mid X_t))]
    \;\le\; \varepsilon_t^2$.
Combining \cref{eq:BB2,eq:B1} gives 
\begin{align}
\label{eq:upper-I}
    \mathrm{I}
\;\leq\;  \frac{1}{P_{\rm data}(y)} \cdot \frac{\varepsilon_t}{\sqrt{2}}.
\end{align}

We next upper bound term $\mathrm{II}$. 
Define $\mathcal{Z} = B \cup \left\{ * \right\}$, and a mapping $\mathcal{A}^B : \mathbb{R}^d \to \mathcal{Z}$ by
\[
    \mathcal{A}^{B}(x) = \begin{cases}
        x &x \in B, \\
        * & x \notin B.
    \end{cases}
\]
Let $P_y = p_t(\cdot \mid y)$ and $\hat P_y = \hat p_{t}(\cdot \mid y)$. 
Define the pushforward measures $P_{y,B} = (\mathcal{A}^B)_\sharp P_y$ and $\hat P_{y,B} = (\mathcal{A}^B)_\sharp \hat P_y$. 
By the data processing inequality for KL divergences, we have
\begin{align*}
    \KL(P_{y, B} \mid\mid \hat{P}_{y, B}) \leq \KL(P_{y} \mid\mid \hat{P}_{y}) = \KL(p_{t}(\cdot \mid y) \mid\mid \hat{p}_{t}(\cdot \mid y)). 
\end{align*}
We then compute $\KL(P_{y,B} \,\|\, \hat P_{y,B})$ explicitly. 
On the set $B$, the measures $P_{y,B}$ and $\hat P_{y,B}$ admit Lebesgue densities $p_t(x \mid y)$ and $\hat p_{t}(x \mid y)$, while at the atom $\{ * \}$ they have masses $P_y(B^c)$ and $\hat P_y(B^c)$, respectively. As a consequence, 
\begin{align*}
    \KL(P_{y, B} \mid\mid \hat{P}_{y, B}) &= \int_{B} p_{t}(x \mid y) \log \left( \frac{p_{t}(x \mid y)}{\hat{p}_{t}(x \mid y)} \right) \dd x + P_{y}(B^{c}) \log \left( \frac{P_{y}(B^{c})}{\hat P_y(B^c)} \right)\\
    &=\int_{B} p_{t}(x \mid y) \log \left( \frac{p_{t}(y \mid x)}{\hat{p}_{t}(y \mid x)} \right) \dd x + P_y(B) \log \left( \frac{\hat p_{t}(y)}{p_t(y)} \right) + P_{y}(B^{c}) \log \left( \frac{P_{y}(B^{c})}{\hat P_y(B^c)} \right), 
\end{align*}
where the second equality uses the decompositions $p_t(x \mid y) = p_t(x) p_t(y \mid x) / p_t(y)$ and $\hat p_{t}(x \mid y) = p_t(x) \hat p_{t}(y \mid x) / \hat p_{t}(y)$ (using the notations from Lemma \ref{lem:posterior-KL-bound}). 
Leveraging the above equality, we have
\begin{align*}
    & \int_{B} p_{t}(x \mid y) \log \left( \frac{p_{t}(y \mid x)}{\hat{p}_{t}(y \mid x)} \right) \dd x \\
    &= \KL(P_{y, B} \mid\mid \hat{P}_{y, B}) - P_{y}(B^{c}) \log \left( \frac{P_{y}(B^{c})}{\hat P_y(B^c)} \right) - P_y(B) \log \left( \frac{\hat p_{t}(y)}{p_t(y)} \right) \\
    &\leq \KL(p_{t}(\cdot \mid y) \mid\mid \hat{p}_{t}(\cdot \mid y)) - P_{y}(B^{c}) \log \left( \frac{P_{y}(B^{c})}{\hat P_y(B^c)} \right) - P_y(B) \log \left( \frac{\hat p_{t}(y)}{p_t(y)} \right).
\end{align*}
By Lemma \ref{lem:posterior-KL-bound}, it holds that $\KL(p_{t}(\cdot \mid y) \mid\mid \hat{p}_{t}(\cdot \mid y)) \leq \frac{\varepsilon_t^2}{P_{\rm data}(y)}$. 
In addition, $\big|\hat P_y(B^c) - P_y(B^c)\big| \le \,\mathrm{TV}(P_y, \hat P_y) \leq \sqrt{\KL(P_y \parallel \hat P_y) / 2} \leq \varepsilon_t  / \sqrt{2 P_{\rm data}(y)}$, and by \cref{eq:B1} we have $\big| p_t(y) - \hat p_{t}(y) \big| \leq \mathbb{E}_{X_t \sim p_t}\big[\,\big|\hat{p}_{t}(y \mid X_{t}) - p_{t}(y \mid X_{t})\big|\,\big] \leq  \varepsilon_t  / \sqrt{2 }$. 
Using these upper bounds and the inequality $\log(1 + x) < x$, we have 
\begin{align}
\label{eq:upper-II}
\begin{split}
\mathrm{II} = & \int_{B} p_{t}(x \mid y) \log \left( \frac{p_{t}(y \mid x)}{\hat{p}_{t}(y \mid x)} \right) \dd x  \\
\leq & \frac{\varepsilon_t^2}{P_{\rm data}(y)} + P_y(B^c) \cdot \frac{|P_y(B^c) - \hat P_y(B^c)|}{P_y(B^c)} + \frac{|p_t(y) - \hat p_t(y)|}{p_t(y) - |p_t(y) - \hat p_t(y)|} \\
\leq & \frac{\varepsilon_t^2}{P_{\rm data}(y)} + \frac{ \varepsilon_t}{\sqrt{2} P_{\rm data}(y)^{1 / 2}} + \frac{\varepsilon_t}{\sqrt{2} (P_{\rm data}(y) - \varepsilon_t / \sqrt{2}) } \\
\leq & \frac{\varepsilon_t^2}{P_{\rm data}(y)} + \frac{ \varepsilon_t}{\sqrt{2} P_{\rm data}(y)^{1 / 2}} + \frac{\sqrt{2}\varepsilon_t}{ P_{\rm data}(y) }.
\end{split}
\end{align}
The proof is done by putting together \cref{eq:upper-I,eq:upper-II}.

\subsection{Proof of Lemma~\ref{lemma:benton}}
\label{proof-benton}
    By Bayes' rule,
    $\nabla_x \log p_s(y \mid x)
    = \nabla_x \log p_s(x \mid y) - \nabla_x \log p_s(x)$.
    Therefore,
    \begin{align*}
        & \sum_{k = 0}^{N - 1} \int_{t_k}^{t_{k + 1}} 
        \E_{x \sim p_{T-t}(\cdot \mid y)}
        \Big[ 
        \big\| 
        \nabla_x \log p_{T - t}(x) 
        + \gamma \nabla_x \log p_{T - t}(y \mid x) 
        - \nabla_x \log p_{T - t_k}(x) 
        - \gamma \nabla_x \log p_{T - t_k}(y \mid x) 
        \big\|_2^2 
        \Big] \dd t \\
        &= \sum_{k = 0}^{N - 1} \int_{t_k}^{t_{k + 1}} 
        \E_{x \sim p_{T-t}(\cdot \mid y)}
        \Big[
        \big\|
        (1-\gamma)
        \big(
        \nabla_x \log p_{T-t}(x)
        -
        \nabla_x \log p_{T-t_k}(x)
        \big)
        +
        \gamma
        \big(
        \nabla_x \log p_{T-t}(x \mid y)
        -
        \nabla_x \log p_{T-t_k}(x \mid y)
        \big)
        \big\|_2^2
        \Big] \dd t \\
        &\lesssim
        \big((1-\gamma)^2+\gamma^2\big)
        \Big(
        \kappa^2 (d + R^2) N 
        + \kappa (d + R^2) T
        \Big) \\
        &=
        (2\gamma^2-2\gamma+1)
        \Big(
        \kappa^2 (d + R^2) N 
        + \kappa (d + R^2) T
        \Big),
    \end{align*}
    where the second inequality follows by applying
    $\|a+b\|_2^2 \lesssim \|a\|_2^2+\|b\|_2^2$ and the same argument as \cite[Lemma~2]{benton2024nearly} to the unconditional and conditional score increments separately.

\subsection{Proof of Lemma~\ref{lemma:initialization}}
\label{proof-intialization}

Let $h_T(x):=p_T(y\mid x)$ and $p_y:=P_{\mathrm{data}}(y)$. Recall that
$p_{T,\gamma}(x)/p_T(x)=h_T(x)^\gamma/Z_{T,\gamma}$ and, by Jensen's inequality, $Z_{T,\gamma}\geq p_y^\gamma$. Using the inequality $D_{\mathrm{KL}}(P\|Q)\leq\chi^2(P\|Q)$, we obtain
\begin{align*}
D_{\mathrm{KL}}\big(p_{T,\gamma}\,\|\,p_T\big)
&\leq
\frac{\Var_{X\sim p_T}\big(h_T(X)^\gamma\big)}
     {Z_{T,\gamma}^2} \leq
\frac{\gamma^2}{p_y^{2\gamma}}
\E_{X\sim p_T}\big[(h_T(X)-p_y)^2\big],
\end{align*}
where the second inequality follows since $u\mapsto u^\gamma$ is
$\gamma$-Lipschitz on $[0,1]$.

Let $Y_y:=\mathbbm{1}\{Y=y\}$. Since
$h_T(X_T)=\mathbb{P}(Y_y=1\mid X_T)$, Pinsker's inequality and the data-processing
inequality give
$\E[(h_T(X_T)-p_y)^2]\leq \frac12 I(Y_y;X_T)
\leq \frac12 I(X_0;X_T)$.
Writing the forward OU process as
$X_T=e^{-T}X_0+\sqrt{1-e^{-2T}}\,Z$, where $Z\sim N(0,I_d)$, and using
$\|X_0\|_2\leq R$ almost surely, the standard Gaussian-channel bound yields
\begin{equation*}
I(X_0;X_T)
\leq
\frac{d}{2}
\log\left(
1+\frac{e^{-2T}R^2}{d(1-e^{-2T})}
\right)
\leq
\frac{e^{-2T}R^2}{2(1-e^{-2T})}.
\end{equation*}
Consequently,
\[
D_{\mathrm{KL}}\big(p_{T,\gamma}\,\|\,p_T\big)
\leq
\frac{\gamma^2R^2e^{-2T}}
{4p_y^{2\gamma}(1-e^{-2T})}.
\]

It remains to compare $p_T$ with $\pi_d$. Applying the Donsker--Varadhan variational inequality with $P=p_{T,\gamma}$, $Q=p_T$, and $f(x)=\log\big(p_T(x)/\pi_d(x)\big)$,
\begin{align*}
D_{\mathrm{KL}}\big(p_{T,\gamma}\,\|\,\pi_d\big)
&=
D_{\mathrm{KL}}\big(p_{T,\gamma}\,\|\,p_T\big)
+
\E_{p_{T,\gamma}}
\left[
\log\frac{p_T(X)}{\pi_d(X)}
\right] \\
&\leq
2D_{\mathrm{KL}}\big(p_{T,\gamma}\,\|\,p_T\big)
+
\log\int_{\R^d}\frac{p_T(x)^2}{\pi_d(x)}\,\dd x.
\end{align*}
Using again the OU representation, let $a:=e^{-T}$ and
$\sigma_T^2:=1-a^2$. Then
$p_T(x)=\E_{X_0}\big[\phi_{\sigma_T^2 I_d}(x-aX_0)\big]$, where
$\phi_{\Sigma}$ denotes the density of $N(0,\Sigma)$. Thus, for independent $X_0,X_0'\sim P_{\rm data}$, we have $(X_0,X_0')\sim P_{\rm data}\otimes P_{\rm data}$. This allows us to express $p_T(x)^2$ as an expectation over two independent copies, yielding:
\[
\int_{\R^d}\frac{p_T(x)^2}{\pi_d(x)}\,\dd x
=
\E_{X_0,X_0'}
\left[
\int_{\R^d}
\frac{
\phi_{\sigma_T^2 I_d}(x-aX_0)
\phi_{\sigma_T^2 I_d}(x-aX_0')
}{
\phi_{I_d}(x)
}
\,\dd x
\right].
\]
Completing the square, for fixed $u,v\in\R^d$, the inner integral is
$$\int_{\R^d}
\frac{
\phi_{\sigma_T^2 I_d}(x-au)
\phi_{\sigma_T^2 I_d}(x-av)
}{
\phi_{I_d}(x)
}
\,\dd x = (1-a^4)^{-d/2}
\exp\!\left\{
\frac{a^2}{2(1-a^4)}
\left(2u^\top v-a^2\norm{u}_2^2-a^2\norm{v}_2^2\right)
\right\}.$$
Since $\norm{u}_2,\norm{v}_2\leq R$ and
$2u^\top v\leq\norm{u}_2^2+\norm{v}_2^2$, we have
$2u^\top v-a^2\norm{u}_2^2-a^2\norm{v}_2^2
\leq2(1-a^2)R^2$. Therefore,
\[
\int_{\R^d}\frac{p_T(x)^2}{\pi_d(x)}\,\dd x
\leq
(1-a^4)^{-d/2}
\exp\left\{
\frac{a^2R^2}{1+a^2}
\right\}.
\]
Taking logarithms and substituting $a=e^{-T}$ yields
\[
\log\int_{\R^d}\frac{p_T(x)^2}{\pi_d(x)}\,\dd x
\leq
-\frac{d}{2}\log(1-e^{-4T})
+
\frac{R^2e^{-2T}}{1+e^{-2T}}.
\]
Since $T\geq1$, we have
$(1-e^{-2T})^{-1}\lesssim1$ and
$-\log(1-e^{-4T})\lesssim e^{-4T}\leq e^{-2T}$. Combining the preceding
bounds yields
\begin{equation*}
D_{\mathrm{KL}}\big(p_{T,\gamma}\,\|\,\pi_d\big)
\lesssim
\left(
d+\frac{\gamma^2R^2}{P_{\mathrm{data}}^{2\gamma}}
\right)e^{-2T},
\end{equation*}
which proves the result.

\subsection{Additional numerical experiments}
\label{appendix:end-to-end-guidance-exp}

\textcolor{black}{
This section presents additional numerical results for the neural classifier experiment described in Section~\ref{subsec:end-to-end-guidance-exp}. 
Figure~\ref{fig:training-curves} shows the training loss and accuracy curves, indicating that both classifiers achieve similar performance. 
}

\begin{figure}[htbp!]
    \centering
    \includegraphics[width=\textwidth]{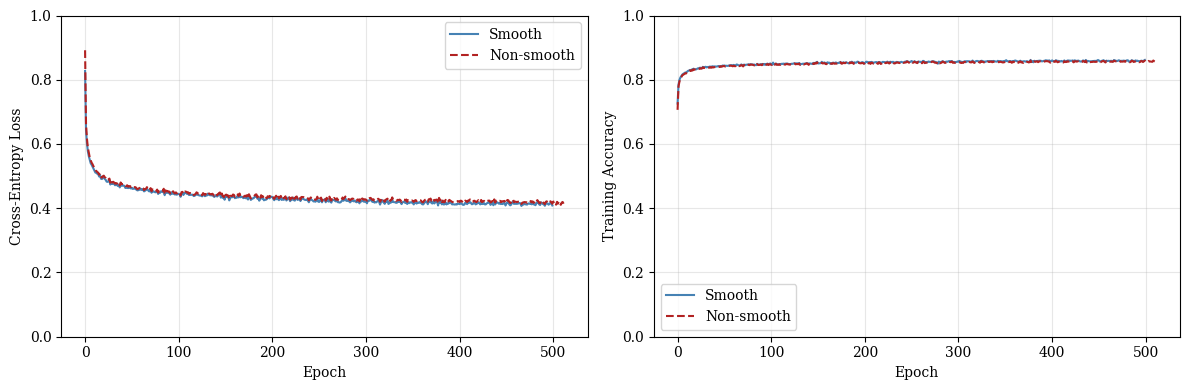}
    \caption{Training loss and accuracy curves for smooth and non-smooth classifiers. Both achieve similar final performance in terms of cross-entropy and accuracy.}
    \label{fig:training-curves}
\end{figure}



\end{document}